\documentclass{article}

\usepackage{Jasonbeamer}
\usepackage{authblk}
\usepackage{caption}
\usepackage{comment}

\graphicspath{ {./image/} }

\usepackage[textwidth=2.0cm, textsize=tiny]{todonotes} 

\begin{document}
\title{G$^2$N$^2$ : Weisfeiler and Lehman go grammatical}
\date{}
\author[1]{Jason Piquenot}
\author[1]{Aldo Moscatelli}
\author[1]{Maxime B\'erar}
\author[1]{Pierre H\'eroux}
\author[2]{Jean-Yves Ramel}
\author[2]{Romain Raveaux}
\author[1]{S\'ebastien Adam}
\affil[1]{LITIS Lab, University of Rouen Normandy, France}
\affil[2]{LIFAT Lab, University of Tours, France}
\maketitle

\begin{abstract}

This paper introduces a framework for formally establishing a connection between a portion of an algebraic language and a Graph Neural Network (GNN). The framework leverages Context-Free Grammars (CFG) to organize algebraic operations into generative rules that can be translated into a GNN layer model. As CFGs derived directly from a language tend to contain redundancies in their rules and variables, we present a grammar reduction scheme. By applying this strategy, we define a CFG that conforms to the third-order Weisfeiler-Lehman (3-WL) test using MATLANG. From this 3-WL CFG, we derive a GNN model, named G$^2$N$^2$, which is provably 3-WL compliant. Through various experiments, we demonstrate the superior efficiency of G$^2$N$^2$ compared to other 3-WL GNNs across numerous downstream tasks. Specifically, one experiment highlights the benefits of grammar reduction within our framework.
\end{abstract}

\section{Introduction}\label{sec: 1}


In the last few years, the Weisfeiler-Lehman (WL) hierarchy, based on the eponymous polynomial-time isomorphism test (\cite{wesfleman}), has been the most common way to characterise the expressive power of Graph Neural Networks (GNNs) (\cite{morris2019wl,bodnar2021weisfeiler,bodnar2021weisfeiler2,Zhang2023ACE}). A founding result was the proof that Message Passing Neural Networks (MPNNs) (\cite{gilmer2017neural,wu2020comprehensive}) are at most as powerful as the first-order WL test (1-WL) (\cite{morris2019wl,xu2018powerful}). As a consequence of this result, many subsequent contributions have focused on going beyond this $\WL 1$ limit, to reach more expressive GNNs. For instance, subgraph-based GNNs  (\cite{chen2020can,zhang2021nested,zhao2021stars}) succeed to surpass $\WL 1$ expressive power but are still bounded by $\WL 3$ (\cite{frasca2022understanding}). 


One way to ensure $\WL k$ expressive power is to mimic one iteration of the $\WL k$ test (\cite{maron2019provably}) for each GNN layer.
Taking as root the colouring and hashing steps of the $\WL k$ algorithm, \cite{maron2019provably} shows that $k$-IGN, based on the basis of equivariant operators defined for IGN (\cite{maron2018invariant}), is as powerful as the $\WL k$ test. Since $k$-IGN works on $k$-th order tensors and since the cardinal of the basis is equal to the $2k$-th Bell number, it is limited in practice by both the layer input memory consumption and the cardinal of IGN operator basis, even for $k=3$ (\cite{NEURIPS2020_2f73168b}). Concurrently, Provably Powerful Graph Network (PPGN) was also proposed in \cite{maron2019provably}. It is able to mimic the second-order Folklore WL test ($\FWL 2$\footnote{known to be equivalent to $\WL 3$ test (\cite{huang2021short})}) colouring and hashing steps with MLPs that are coupled together with matrix multiplication. Since PPGN only relies on matrices, it is a more tractable $\WL 3$ architecture than 3-IGN (\cite{zhang2023expressive}).

Taking an algebraic point of view, the groundbreaking paper \cite{Geerts} reformulates the $\WL 1$ and $\WL 3$ tests as languages based on specific subsets of algebraic operations applied on the adjacency matrix. These fragments of the matrix language MATLANG (\cite{matlang}) called $\ML{\cur L_1}$ and $\ML{\cur L_3}$ are shown to be as expressive as $\WL 1$ and $\WL 3$ (\cite{Geerts}). Derived from this result, a model called GNNML1 was proposed in \cite{balcilarexpresspower}. GNNML1 is proven to be $\WL 1$ equivalent since it is able to generate any sentence of $\ML{\cur L_1}$. A more expressive model called GNNML3 was proposed in the same paper. It is only shown to be more expressive than $\WL 1$. This is due to the lack of a systematic procedure of deriving a GNN model from a given language fragment.



In this paper, we leverage this bottleneck by proposing a generic methodology to produce a GNN from any fragment of an algebraic language, opening a new way to ensure expressiveness. The rationale behind our framework is to instantiate a language fragment by a reduced set of generative rules, translated into layer components of a GNN. Starting from the operations set $\cur L_3$, we build an exhaustive Context-Free Grammar (CFG) able to generate $\ML{\cur L_3}$. This CFG is reduced to remove unnecessary operations among the rules while keeping the equivalence with $\WL 3$. From the variables of this reduced CFG, GNN inputs are easily deduced. Then, the rules of the CFG determine the GNN layers update functions. As a result of this methodology, we propose a new model called Grammatical Graph Neural Network (G$^2$N$^2$) that is provably $\WL 3 $.

  
The contributions of this work are the following : (i) A generic framework to design a GNN from any fragment of an algebraic language; (ii) The instantiation of the framework on $\ML{\cur L_3}$ resulting in G$^2$N$^2$, a provably $\WL 3 $ GNN; \textbf(iii); An experimental validation of the set of rules; (iv) Numerous experiments demonstrating that G$^2$N$^2$ outperforms existing $\WL 3$ GNNs on various downstream tasks.

The paper is structured as follows. Section \ref{sec: 2} introduces the necessary background, by defining MATLANG, its link with WL and CFGs. Section \ref{sec: 3} describes our framework and presents the resulting G$^2$N$^2$ architecture, which is experimentally evaluated in section \ref{sec: 4}.

\section{From MATLANG and Weisfeiler-Lehman to Context-Free Grammars and Languages}\label{sec: 2}

Let $\cur G = (\cur V, \cur E)$ be an undirected graph where $\cur V = \intervalleentier 1 n$ is the set of $n$ nodes and $\cur E \inclu \cur V \fois \cur V$ is the set of edges. The adjacency matrix $A \dans \{0,1\}^{n\fois n}$ represents the connectivity of $\cur G$.

\begin{defin}[MATLANG (\cite{matlang})]

MATLANG is a matrix language with an allowed operation set $\{+,\cdot, \odot ,\transpose  \ ,\mathrm{Tr},  \mathrm{diag}, \onevector ,  \fois, f \}$ denoting respectively matrix addition, matrix and element-wise multiplications, transpose and trace computations, diagonal matrix creation from a vector, column vector of $1$ generation, scalar multiplication, and element-wise function applied on a scalar, a vector or a matrix. Restricting the set of operations to a subset $\cur L$ defines a fragment of MATLANG denoted $\ML {\cur L}$. $s(X) \dans \R$ is a sentence in $\ML{\cur L}$ if it consists of consistent consecutive operations in $\cur L$, operating on a given matrix $X$, resulting in a scalar value. \textit{As an example, $s(X) = \transpose \onevector \left( X^2 \odot \diag \onevector \right) \onevector$ is a sentence of $\ML{ \{ \cdot , \transpose \ , \onevector , \mathrm{diag} , \odot \} }$ computing the trace of $X^2$.}
\end{defin}

Equivalences between $\ML{\cur L_1}$ and $\ML{\cur L_3}$ with $\cur L_1 = \{\cdot , \transpose \ , \onevector , \mathrm{diag} \}$, $\cur L_3 = \{\cdot , \transpose \ , \onevector , \mathrm{diag} , \odot \}$ and respectively the $\WL 1$ and $\WL 3$ tests are shown in \cite{Geerts}: two graphs are indistinguishable by the $\WL 1$ (resp. $\WL 3$) test if and only if applying any sentence of $\ML{\cur L_1}$ (resp. $\ML{\cur L_3}$) to their adjacency matrices gives the same scalar. Adding $\{+, \times,f\}$ does not improve the expressive power of the fragment (\cite{Geerts}).



Transposed in a Machine Learning context, a MATLANG-based GNN will inherit the $\WL 3$  expressive power of $\ML{\cur L_3}$ if it is able to generate any sentence of the fragment while learning the downstream task. To reach this objective, we will instantiate the fragment as a Context Free Language, entirely described by a set of production rules\footnote{Figure \ref{fig:fullgraph} in appendix \ref{app:CFG} illustrates the process of sentence generation from a grammar.}. 

\begin{defin}[Context-Free Grammar and Language]\label{def:CFG}
A Context-Free Grammar (CFG) $G$ is a 4-tuple $(V,\Sigma, R, S)$ with $V$ a finite set of variables, $\Sigma$ a finite set of terminal symbols, $R$ a finite set of rules $V \vers \left(V\union \Sigma \right)^*$, $S$ a start variable. \textit{$R$ completely describes a CFG with the convention that $S$ is placed on the top left.}

$B$ is a Context-Free Language (CFL) if there exists a CFG $G$ such that $B = L(G) := \{w,w \dans \Sigma^* \text{ and } S \derive w\}$ where $S \derive w$ denotes that $S$ can be transformed into $w$ by applying an arbitrary number of rules in $G$.
\end{defin}





\section{From $\ML{\cur L_3}$ to the $\WL 3$ G$^2$N$^2$}\label{sec: 3}

In this section, the proposed generic framework is described and instantiated on the $\ML{\cur L_3}$ fragment to generate our G$^2$N$^2$ model. As shown by Figure \ref{fig:CFGtognn}, 3 steps are involved:\\
\textbf{(1) defining the exhaustive CFG that generates the language,} \textbf{(2) reducing the exhaustive CFG,} \textbf{(3) translating the variables and the rules of the reduced CFG into GNN input and model layer.} To keep the expressive power of the language at each step, the equivalence between the successive representations must be ensured.  


\begin{figure}[h]
\centering
\includegraphics[width=.97\textwidth]{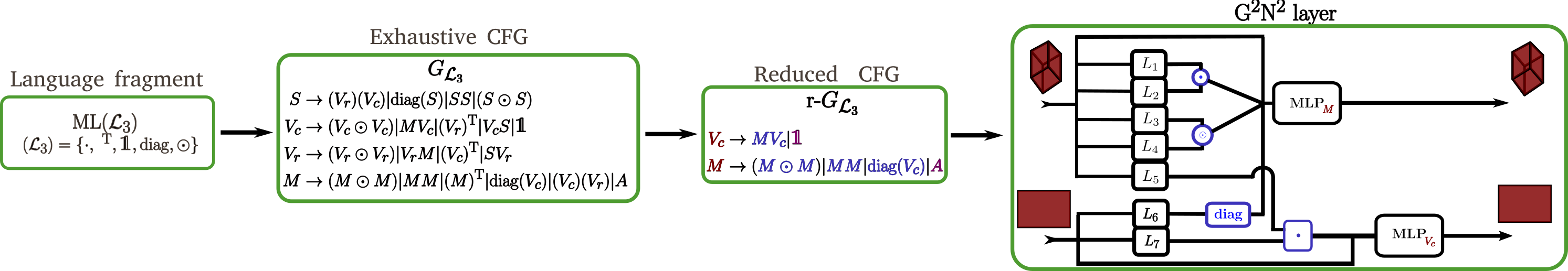}
\caption{\textbf{Overview of the proposed framework instantiated on $\ML{\cur L_3}$.}
}
\label{fig:CFGtognn}
\end{figure}


\subsection{From $\ML{\cur L_3}$ to the exhaustive CFG $G_{\cur L_3}$} 

The first step of the framework translates the language fragment into an exhaustive CFG (variables, terminal symbols and rules). For $\ML{\cur L_3}$, the variables of  the exhaustive CFG denoted $G_{\cur L_3}$ are defined using  the following proposition proved in appendix \ref{subsec:proofl3}.

\begin{prop}\label{prop:ML3state}
For any square matrix of size $n^2$, operations in ${\cur L_3}$ can only produce square matrices of the same size, row, or column vectors of size $n$ or scalars. 
\end{prop}

In the context of our study, as in \cite{Geerts}, $\ML{\cur L_3}$ is applied on the  adjacency matrix. Thus, proposition \ref{prop:ML3state} ensures that $G_{\cur L_3}$ variables are restricted to square matrix ($M$), column vector ($V_c$), row vector ($V_r$) and scalar ($S$). Once the variables defined, the production rules of $G_{\cur L_3}$ are obtained by enumerating all possible operations in $\ML{\cur L_3}$ that produce such variables. The rule $M \vers A$ is added in order to be compliant with \cite{Geerts}. All the rules composing $G_{\cur L_3}$ are synthesised in equation \ref{eq:gl3} where $|$ denotes the classical OR operator since a variable can be produced by different rules. They fully characterise the CFG\footnote{Elements that are not variables in the rule set are said to be terminal symbols.}.

The following theorem ensures that the language generated by $G_{\cur L_3}$ is $\ML{\cur L_3}$. Thus $G_{\cur L_3}$ is as expressive as $\ML{\cur L_3}$.

\begin{thm}\label{thm: exhaustivegl3}
For $G_{\cur L_3}$ defined by
\begin{align}\label{eq:gl3}
S &\vers (V_r)(V_c) \ | \ \diag S \ | \ SS \ | \ (S \odot S) \\
V_c &\vers (V_c \odot V_c) \ | \ MV_c \ | \ \transpose{(V_r)} \ | \ V_c S \ | \ \onevector \nonumber \\
V_r &\vers (V_r \odot V_r) \ | \ V_r M \ | \ \transpose{(V_c)} \ | \ S V_r \nonumber \\
M &\vers (M\odot M) \ | \ MM \ | \ \transpose{(M)} \ | \ \diag {V_c} \ | \ (V_c)(V_r) \ | \ A \nonumber
\end{align}
we have$$L(G_{\cur L_3}) =\ML{\cur L_3}.$$

\end{thm}
The full proof is provided in appendix (\ref{subsec:proofl3}). Its idea is the following. As any operation in the rules of $G_{\cur L_3}$ belongs to $\cur L_3$, it is clear that $L(G_{\cur L_3}) \inclu \ML {\cur L_3} $. The reciprocal inclusion is proven by induction over the number of $\ML{\cur L_3}$ operations.

Given the results of theorem \ref{thm: exhaustivegl3}, the next step reduces the CFG by exploiting the redundancies in the exhaustive set of rules and variables.

\subsection{From $G_{\cur L_3}$ to r-$G_{\cur L_3}$}

An example of redundancy can be observed in the following proposition proved in the appendix (see \ref{subsec:proofl3}).

\begin{prop}\label{prop:vectorodot}
For any square matrix $M$, column vector $V_c$ and row vector $V_r$, we have
\begin{align*}
M \odot (V_c\cdot V_r)  &= \diag {V_c} M \diag {V_r} 
\end{align*} 
\end{prop} 

The following theorem guarantees that the following reduced grammar preserves expressiveness.

\begin{thm}[$\ML{\cur L_3}$ reduced CFG ]\label{thm : ML3reduced}
Let r-$G_{\cur L_3}$ be defined by
\begin{align}\label{eq:reducedCFG3}
V_c &\vers  MV_c \ | \ \onevector  \\ \nonumber
M &\vers (M\odot M) \ | \ MM \ | \ \diag {V_c} \ | \ A 
\end{align}

r-$G_{\cur L_3}$ is as expressive as $G_{\cur L_3}$.

\end{thm}

\begin{proof}

For any scalar $S,S'$, since $\diag S$, $S \odot S'$ and $S\cdot S'$ produce a scalar, the only way to produce a scalar from other variables is to pass through a vector dot product. Hence the scalar variable $S$ and its rules can be removed from $G_{\cur L_3}$ without loss of expressive power.

Since $\diag v w = v\odot w$ for any vector $v,w$, the vector Hadamard product can be removed from the vector rules. Proposition \ref{prop:vectorodot} allows to remove $V_c V_r$ from the rules of $M$
since the results of subsequent mandatory operations $MM$ or $MV_c$ can be obtained with other combinations. At this stage, the following intermediate CFG i-$G_{\cur L_3}$ is as expressive as $G_{\cur L_3}$ since it can compute any vector of $G_{\cur L_3}$.
\begin{align*}
V_c &\vers  MV_c \ | \ \transpose{(V_r)} \ | \ \onevector \\
V_r &\vers  V_r M \ | \ \transpose{(V_c)} \\
M &\vers (M\odot M) \ | \ MM \ | \ \transpose{(M)} \ | \ \diag {V_c} \ | \ A 
\end{align*}
Since the remaining $M$ rules preserve symmetry, $\transpose{(M)}$, the variable $V_r$ and its rules can be removed. It conducts to r-$G_{\cur L_3}$ defined in equation \ref{eq:reducedCFG3}. 
\end{proof}

From these two steps, the resulting CFG r-$G_{\cur L_3}$ possesses the expressive power of the fragment $\ML{\cur L_3}$. The next step is a translation of r-$G_{\cur L_3}$ into a GNN layer.

\subsection{From r-$G_{\cur L_3}$ to a G$^2$N$^2$ layer model}\label{subsec:ML1}

\begin{figure}[b!]
    \centering
    \includegraphics[width = \textwidth]{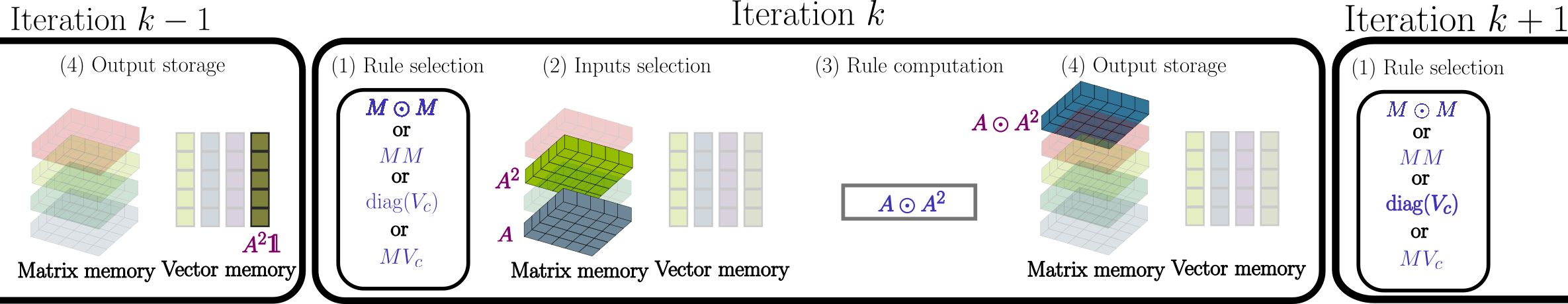}
    \caption{\textbf{4-step iterative procedure} (1) Rule selection (2) Inputs selection: inputs relative to the chosen rule are selected from matrix and/or vector memories (opaque matrices) (3) Rule computation (4) Output storage: the produced output is stored into its relative memory.}

    \label{fig:pdatognn}
\end{figure}

In r-$G(\cur L_3)$, any vector $V_c$ or matrix $M$ is produced by applying a sequence of rules on $A$ and $\onevector$. 
As a consequence, every matrix or vector can be attained through an iterative rule selection procedure using matrix and vector memories that store intermediate variables. 
Figure \ref{fig:pdatognn} describes this procedure: each iteration starts by choosing a rule in r-$G(\cur L_3)$ before selecting corresponding inputs in the memories. Applying the selected rule produces a new matrix or a new vector, which is added to the appropriate memory.


Translating this iterative procedure into a GNN based on a sequence of layers requires a memory management strategy and a selection mechanism for both rules and inputs, while taking into account learning issues related to downstream tasks.

The matrix memory aims at storing the variables \textcolor{Red}{$M$} produced by successive applications of r-$G(\cur L_3)$ rules. This memory is represented by a three order tensor \textcolor{Red}{$\cur C^{(l)}$} where produced matrices (i.e. edges embeddings in a GNN context) are stacked across layers on the third dimension. In the same way, the vector memory is dedicated to the variables \textcolor{Red}{$V_c$}  that correspond to nodes embeddings. It is as a matrix \textcolor{Red}{$H^{(l)}$} where produced vectors are stacked on the second dimension. \textcolor{Red}{$\cur C^{(l)}$} and \textcolor{Red}{$H^{(l)}$} are the input of the $l$-th GNN layer which produces \textcolor{Red}{$\cur C^{(l+1)}$} and \textcolor{Red}{$H^{(l+1)}$}  as output, as depicted in Figure \ref{fig:g2n2layer} describing a G$^2$N$^2$ layer. While the memory of the iterative procedure grows with each iteration, a tractable GNN architecture constrains the stacking dimension to be set to a given value at each layer.

In order to mimic the rule selection procedure of Figure \ref{fig:pdatognn}, a G$^2$N$^2$ layer applies a selection among the outputs produced by all the rules.   
Such a strategy enables to compute in parallel several occurrences of any rule with multiple inputs. Hence, parameterised quantities \textcolor{Green}{$b_\odot$},\textcolor{Green}{$b_\cdot$},\textcolor{Green}{$b_{\mathrm{diag}}$},\textcolor{Green}{$b_{MV_c}$} of the rules  \textcolor{Blue}{$(M\odot M)$}, \textcolor{Blue}{$(MM)$}, \textcolor{Blue}{$\diag {V_c}$}, \textcolor{Blue}{$MV_c$} are computed in parallel taking as input linear combination $L_i$ of slices of \textcolor{Red}{$\cur C^{(l)}$} and slices of \textcolor{Red}{$H^{(l)}$}. These linear combinations are able to select among inputs  \textcolor{Red}{$\cur C^{(l)}$} and  \textcolor{Red}{$H^{(l)}$} through a learning paradigm.  

Both the matrix rules outputs and the tensor \textcolor{Red}{$\cur C^{(l)}$} (obtained through a skip connection which guarantees the memory persistence) are fed to MLP$_M$ that produces the output tensor \textcolor{Red}{$\cur C^{(l+1)}$} with a selected third dimension size $S^{(l+1)}$. This MLP allows in the same time to simulate the rule selection, to compress the matrix output of the layer to a fixed size and to learn a point wise function for solving specific downstream tasks. It relates to the set of operations $\{+,\times , f\}$ of $MATLANG$ and does not modify the expressive power (\cite{Geerts,maron2019provably}).
The output  \textcolor{Red}{$H^{(l+1)}$} is provided similarly through MLP$_{V_c}$. Figure \ref{fig:g2n2layer} describes the whole model of a $\ggnn$ layer.

\begin{figure}[h]
    \centering
    \includegraphics[width = \textwidth]{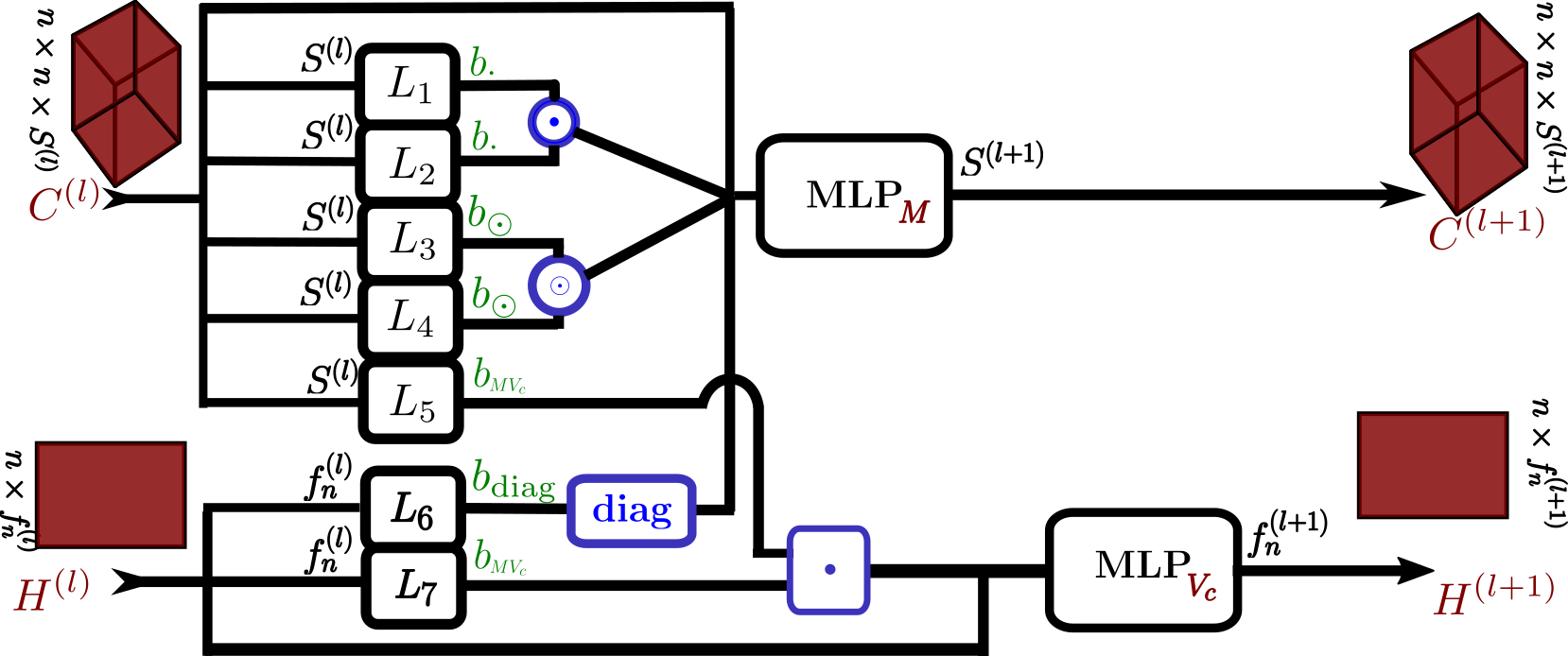}
    \caption{$L_{1}$-$L_5$ combine the $S^{(l)}$ slices of \textcolor{Red}{$\cur C^{(l)}$}  into $2$\textcolor{Green}{$b_\odot$}, $2$\textcolor{Green}{$b_\cdot$} and \textcolor{Green}{$b_{MV_c}$} matrices. $L_6$-$L_7$ combine the $f_n^{(l)}$ columns of \textcolor{Red}{$H^{(l)}$} into \textcolor{Green}{$b_{\mathrm{diag}}$} and \textcolor{Green}{$b_{MV_c}$} vectors. From the outputs of $L_1$ -$L_7$, multiple occurrences of r-$G(\cur L_3)$ rules \textcolor{Blue}{$(M\odot M)$}, \textcolor{Blue}{$(M M)$}, \textcolor{Blue}{$(\mathrm{diag}(V_c)$} and \textcolor{Blue}{$(MV_c)$} are computed. The obtained outputs and the layer inputs are fed to MLP$_M$ and MLP$_{V_c}$ providing the layer outputs \textcolor{Red}{$\cur C^{(l+1)}$} and \textcolor{Red}{$H^{(l+1)}$.}}
\label{fig:g2n2layer}
\end{figure}

Formally, the update equations are :
\begin{align}
\textcolor{Red}{\cur C^{(l+1)}} &= \text{MLP}_{M}(\textcolor{Red}{\cur C^{(l)}}||L_1(\textcolor{Red}{\cur C^{(l)}})\textcolor{Blue}{\cdot} L_2(\textcolor{Red}{\cur C^{(l)}})|| 
L_3(\textcolor{Red}{\cur C^{(l)}}) \textcolor{Blue}{\odot} L_4(\textcolor{Red}{\cur C^{(l)}})||\textcolor{Blue}{\mathrm{diag}}({L_6(\textcolor{Red}{H^{(l)}}))}), \\
\textcolor{Red}{H^{(l+1)}} &= \text{MLP}_{V_c} ( \textcolor{Red}{ H^{(l)}} ||L_5(\textcolor{Red}{\cur C^{(l)}}) \cdot L_7(  \textcolor{Red}{H^{(l)}})), \label{eq:GMNnodeupdate}
\end{align}
where $||$ is the concatenation. MLP$_M$ and MLP$_{V_c}$ are learnable MLPs, and $L_i$ are learnable linear blocks acting on the third dimension of \textcolor{Red}{$\cur C^{(l)}$} or the second dimension of \textcolor{Red}{$H^{(l)}$}: $L_{1,2} : \R^{S^{(l)}} \vers \R^{\textcolor{Green}{b_\cdot^{(l)}}}$, $L_{3,4} : \R^{S^{(l)}} \vers \R^{\textcolor{Green}{b_\odot^{(l)}}}$, $L_{5} : \R^{S^{(l)}} \vers \R^{\textcolor{Green}{b_{MV_c}^{(l)}}}$, $L_{6} : \R^{f^{(l)}} \vers \R^{\textcolor{Green}{b_{\mathrm{diag}}^{(l)}}}$, $L_{7} : \R^{f^{(l)}} \vers \R^{\textcolor{Green}{b_{MV_c}}^{(l)}}$, MLP$_{M} : \R^{S^{(l)}+\textcolor{Green}{b_{\cdot}}^{(l)}+\textcolor{Green}{b_{\odot}}^{(l)}+\textcolor{Green}{b_{\mathrm{diag}}}^{(l)}} \vers \R^{S^{(l+1)}}$, and MLP$_{V_c} : \R^{f^{(l)}+\textcolor{Green}{b_{MV_c}}^{(l)}} \vers \R^{f^{(l+1)}}$. 

\subsection{G$^2$N$^2$ architecture and its expressive power}\label{subsec:g2n2}

Figure \ref{fig:GMNarchi} depicts the global G$^2$N$^2$ architecture. The inputs are \textcolor{Violet}{$H^{(0)}$} and \textcolor{Violet}{$\cur C^{(0)}$}. \textcolor{Violet}{$H^{(0)}$} of size $n \fois f_n + 1$ is the feature nodes matrix concatenated with $\onevector$. \textcolor{Violet}{$\cur C^{(0)}$}$\dans \R^{n \fois n \fois (f_e+1)}$ is a stacking on the third dimension of the adjacency matrix $A$ and the extended adjacency tensor $E$ of size $n \fois n \fois f_e$, where $f_e$ is the number of edge features.

After the last layer, permutation equivariant readout functions are applied on both $H^{(l_{\text{end}})}$ and the diagonal and off-diagonal components of $\cur C^{(l_{\text{end}})}$. Readout outputs are then fed to a dedicated decision layer.

\begin{figure*}[h]
\centering

\includegraphics[width = \textwidth]{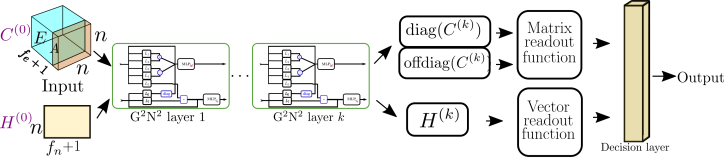}

\caption{\textbf{Model of G$^2$N$^2$ architecture from the graph to the output}. Each layer updates node and edge embeddings and readout functions are applied independently on $H^{(k)}$ and the diagonal and the non-diagonal elements of $\cur C^{(k)}$.}
\label{fig:GMNarchi}
\end{figure*}

\begin{thm}[Expressive power of G$^2$N$^2$]\label{thm: GMN3wl}
G$^2$N$^2$ is able to produce any matrix and vector of L(r-$G_{\cur L_3}$). It is as expressive as $\WL 3$.
\end{thm}
\begin{proof}
We show that G$^2$N$^2$ at layer $l$ can produce all matrices and vectors r-$G_{\cur L_3}$ can produce, after $l$ iterations. It is true for $l=1$. Indeed, at r-$G_{\cur L_3}$ first iteration, we obtain the matrices $\identite$, $A$, $A^2$ and the vectors $\onevector$ and $A\onevector$. Since any of $L_{i}(\cur C^{(0)})$ for $i \dans \intervalleentier 1 5$ is a linear combination of $A$ and $\identite$, G$^2$N$^2$ can produce those vectors and matrices in one layer.

Suppose that there exists $l>0$ such that G$^2$N$^2$ can produce any of the matrices and vectors r-$G_{\cur L_3}$ can after $l$ iterations. We denote by $\cur A_l$ the set of those matrices and by $\cur V_l$ the set of those vectors. At the $l+1$-th iteration, we have $\cur A_{l+1} = \{M\odot N, MN, \diag{V_c} |  M,N \dans \cur A_l\,V_c \dans \cur V_l\}$ and $V_{l+1} = \{MV_c |M\dans \cur A_k,V_c \dans \cur V_l\}$. Let $M,N \dans \cur A_l$ and $V_c \dans \cur V_l$ then by hypothesis G$^2$N$^2$ can produce $M,N$ at layer $l$.
Since $L$ produces at least two different linear combinations of matrices or vectors in respectively $\cur A_l$ and $\cur V_l$, $MN$, $M \odot N$, $MV_c$ and $\diag{ V_c}$ are reachable at layer $l+1$. Thus $\cur A_{l+1}$ is included in the set of matrices G$^2$N$^2$ can produce at layer $l+1$ and $V_{l+1}$ is included in the set of vectors G$^2$N$^2$ can produce at layer $l+1$.
\end{proof}



\subsection{Discussion : $\ggnn$ in the $\WL 3$ GNN literature}


\paragraph{Positioning w.r.t \cite{maron2019provably}}

From PPGN layer description (see Figure 2 of \cite{maron2019provably}), one can build the following CFG:
\begin{align}\label{eq:reducedCFGPPGN}
M &\vers  MM \ | \ \diag {\onevector} \ | \ A 
\end{align}
where $M \vers  \diag{\onevector}$ and $M \vers  A$ represent inputs of the architecture as for G$^2$N$^2$. Compared to r-$G_{\cur L_3}$, $V_c$ variable and $M \vers M\odot M$, $\diag{V_c}$ and $MV_c$ rules are missing. As a consequence, PPGN $\WL 3$ expressive power is not formally inherited from $\ML{\cur L_3}$. As stated in the introduction, it relies on PPGN ability to mimic $\FWL 2$ colouring and hashing steps. Its capacity to  implement the colouring step relies on MLP universality. It explains that PPGN can approximate the missing rules of r-$G_{\cur L_3}$. To guarantee such  an approximation, a certain width and depth for MLP are needed. $\ggnn$ does not suffer from these computational constraints since it only needs to provide linear combinations as arguments of the operations.



3-IGN processes on sets of third order tensors. As a consequence, it cannot be described by a CFG derived from $\ML{\cur L_3}$. However, we can connect our approach with $k$-IGN. For $k$-IGN, the expressive power is related to MLPs and to the basis of linear equivariant operators defined in \cite{maron2018invariant}. In some ways, these operators can be linked to the algebraic operations of our framework. An example of such a link is given in appendix \ref{subsec:cfgarch} for 2-IGN .

\paragraph{Positioning w.r.t \cite{breakingthelimits}}

In appendix \ref{subsec:cfgml1}, we show that GNNML1 (\cite{breakingthelimits}) can be seen as the resulting GNN of our framework applied on $\ML{\cur L_1}$.  
Concerning GNNML3, a CFG can also be deduced from its layer
$$
V_c \vers C_1 V_c \ | \ \cdots \ | \
C_k V_k \ | \ V_c \odot V_c \ | \ \onevector
$$
where the matrices $C_1, \cdots , C_k$ are defined using the adjacency matrix, exponential pointwise function, and matrix Hadamard product. 
As some rules and variables are missing compared to r-$G_{\cur L_3}$, it cannot formally inherit the expressive power of $\ML{\cur L_3}$.

\section{Experiments}\label{sec: 4}

This section is dedicated to the experimental validation of both the framework and $\ggnn$. It answers 4 questions \textbf{Q1}-\textbf{Q4}. \textbf{Q1} concerns the validation of the reduced grammars. \textbf{Q2} and \textbf{Q3} relate to performance of $\ggnn$ on downstream regression/classification tasks. \textbf{Q4} concerns the model spectral ability. Experimental settings are detailed in appendix \ref{subsec:setting}.

\subsection*{\textbf{Q1}: Is the reduction of grammar relevant and optimal?}

This experiment aims at investigating the impact of the CFG reduction scheme through the comparison of different models built using $G_{\cur L_3}$ (see Figure \ref{fig:exhaust_gnn} in appendix \ref{app:CFG}),  i-$G_{\cur L_3}$ (see Figure \ref{fig:interm_gnn} in appendix \ref{app:CFG}) and r-$G_{\cur L_3}$ (see Figure \ref{fig:g2n2layer}). The comparison is completed by an ablation study the aim of which is to investigate the importance of each rule of r-$G_{\cur L_3}$.

We use a graph regression benchmark called QM9 which is composed of 130K small molecules (\cite{ramakrishnan2014quantum,wu2018moleculenet}). For this study, we focus on the regression target $R^2$, which is known to be the most difficult to predict. As in \cite{maron2019provably}, the dataset is randomly split into training, validation, and test sets with a respective ratio of $0.8$, $0.1$ and $0.1$. The edge and vector embeddings are always of size $32$.

The results are presented in Figure \ref{fig:ablation} where each model is represented in a 2-D space using the Mean Absolute Error (MAE) of the best validation model on the test set and the number of parameters of the model. These results corroborate the theoretical results of section \ref{sec: 3} : the MAE scores are comparable for $G(\cur L_3)$, i-$G(\cur L_3)$ and r-$G(\cur L_3)$ while the number of parameters is divided by 2 when reducing from $G(\cur L_3)$ to r-$g(\cur L_3)$.
As expected, removing rules in r-$G(\cur L_3)$ leads to a drop of MAE performance. It also offers insights into the weights of each operation in the model and enables informed pruning of the GNN if the expressiveness is not required.

\begin{figure}[h!]
    \centering
    \includegraphics[width=.6\textwidth]{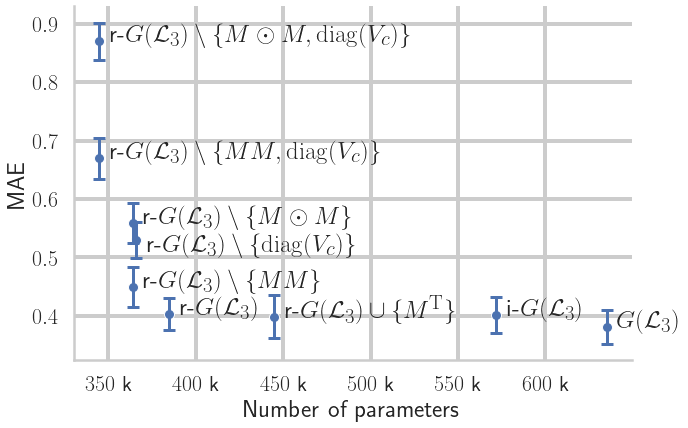}
    \caption{Comparison of model size and MAE performance on the QM9 $R^2$ target for GNNs derived from $G(\cur L_3)$. Each GNN model is denoted by its set of rules. Over-reduced grammar from r-$G(\cur L_3)$ are denoted with a $\setminus$, whereas $\cup$ denotes the addition of a rule to the set.}
    \label{fig:ablation}
\end{figure}



\subsection*{\textbf{Q2}: Does G$^2$N$^2$ perform better than other $\WL 3$ GNNs for regression?}

For this second question, we also use the dataset QM9, but we consider the $12$ regression targets. The dataset is randomly split into training, validation, and test sets with the same ratio as in \textbf{Q1}. The experimental settings are detailed in appendix \ref{subsec:setting}.
G$^2$N$^2$ results are compared to those in \cite{huang2022boosting, maron2019provably} including 1-GNN and 1-2-3-GNN (\cite{morris2019wl}), DTNN (\cite{wu2018moleculenet}), DeepLRP (\cite{chen2020can}), NGNN (\cite{zhang2021nested}), I$^2$-GNN (\cite{huang2022boosting}) and PPGN \cite{maron2019provably}. The metric is the MAE of the best validation model on the test set. The mean epoch duration is measured on the same device for comparison between $\ggnn$ and PPGN.

As in \cite{maron2019provably}, we made two experiments. The first one consists in learning one target at a time while the second learns every target at once. In the first experiment, we have $S^{(l)}=f_n^{(l)} = 64$ and in the second $S^{(l)}=f_n^{(l)} = 32$. Partial results focusing on the two best models are given in Table \ref{tab:QM9}. Complete results and experiment settings are given in appendix \ref{subsec:setting}. In both cases, G$^2$N$^2$ obtains the best results while learning faster.
\begin{table}[h!]
    \centering
    \caption{Results on QM9 dataset focusing on the best methods. On the left part, each target is learned separately while on the right side all targets are learned at the same time. The metric is MAE, the lower, the better. Complete results can be found in Table \ref{tab:QM9completeone}.}\label{tab:QM9}
    \begin{filecontents*}{data/EXPQM9GMNreduct.txt}
Target,PPGN,G$^2$N$^2$
$\mu$,0.0934,\textbf{0.0703}
$\alpha$,0.318,\textbf{0.127}
$\epsilon_{\mathrm{homo}}$,{0.00174},\textbf{0.00172}
$\epsilon_{\mathrm{lumo}}$,{0.0021},\textbf{0.00153}
$\Delta\epsilon$,{0.0029},\textbf{0.00253}
$R^2$,3.78,\textbf{0.342}
ZPVE,{0.000399},\textbf{0.0000951}
$U_0$,{0.022},\textbf{0.0169}
$U$,{0.0504},\textbf{0.0162}
$H$,{0.0294},\textbf{0.0176}
$G$,{0.024},\textbf{0.0214}
$C_v$,0.144,\textbf{0.0429}
T $/$ ep,129 s, 98 s
\end{filecontents*}
\csvautobooktabular{data/EXPQM9GMNreduct.txt} $ \qquad \qquad $
    \begin{filecontents*}{data/EXPQM12.tex}
PPGN,G$^2$N$^2$
0.231,\textbf{0.102}
0.382,\textbf{0.196}
{0.00276},\textbf{0.0021}
{0.00287},\textbf{0.00211}
{0.0029},\textbf{0.00287}
16.07,\textbf{1.19}
{0.00064},\textbf{0.0000151}
{0.234},\textbf{0.0502}
{0.234},\textbf{0.0503}
{0.229},\textbf{0.0503}
{0.238},\textbf{0.0504}
0.184,\textbf{0.0707}
131 s, 57 s
\end{filecontents*}
\csvautobooktabular{data/EXPQM12.tex}
    
\end{table}


\subsection*{
\textbf{Q3}: Does G$^2$N$^2$ perform better than other $\WL 3$ GNNs for classification?
}

For graph classification, we evaluate G$^2$N$^2$ on the classical TUD benchmark (\cite{Morris+2020}), using the evaluation protocol of \cite{xu2018powerful}. Results of GNNs and Graph Kernel are taken from \cite{bouritsas2022improving}. Since the number of node and edge features is very different from one dataset to another, the parameter settings for each of the 6 experiments related to these datasets can be found in Table \ref{tab:parmtud} of appendix \ref{subsec:setting}. Partial results focusing on G$^2$N$^2$ performance are given in Table \ref{tab:TUDd}. Complete results can be seen in Table \ref{tab:TUDcomplete} of appendix \ref{subsec:setting}. G$^2$N$^2$ achieves better than rank 2 for five of the six datasets.

\begin{table}[h]

\centering
    \caption{Results of G$^2$N$^2$ on TUD dataset compared to the best GNN competitor. The rank of G$^2$N$^2$ within GNNs is in parentheses. The metric is accuracy, the higher, the better. Complete results can be seen in Table \ref{tab:TUDcomplete}.}
    \footnotesize{
    \begin{filecontents*}{data/TUDreduct.txt}
Dataset,G$^2$N$^2$,rank,Best GNN competitor
MUTAG,92.5$\pm$5.5,1(1),92.2$\pm$7.5 
PTC,72.3$\pm$6.3,1(1),68.2$\pm$7.2 
Proteins,80.1$\pm$3.7,1(1),77.4$\pm$4.9
NCI1,82.8$\pm$0.9,5(3),83.5$\pm$2.0 
IMDB-B,76.8$\pm$2.8,2(2),77.8$\pm$3.3 
IMDB-M,54.0$\pm$2.9,2(2),54.3$\pm$3.3 
\end{filecontents*}
\csvautobooktabular{data/TUDreduct.txt}
    }
    \label{tab:TUDd}
\end{table}

\subsection*{
\textbf{Q4}: Can G$^2$N$^2$ learn band-pass filters in the spectral domain?
}
As shown in \cite{balcilarexpresspower}, the spectral ability of a GNN and particularly its ability to operate as band-pass filter is an important property of a model for certain downstream tasks. In order to assess the spectral ability of G$^2$N$^2$ and answer Q4, we use the protocol and node regression dataset of \cite{breakingthelimits}. $R^2$ score is used to compare performance.


Table \ref{tab:filter} reports the comparison of G$^2$N$^2$ to  CHEBNET (\cite{chebnet}), PPGN and GNNML3, citing the results from \cite{breakingthelimits}. CHEBNET and GNNML3 are spectrally designed and manage to learn low-pass, high-pass, and band-pass filters. For the three filter types, G$^2$N$^2$  reaches comparable performance. In appendix \ref{sec:spectral}, a theoretical analysis shows that a $\WL 3$ GNN is able to approximate any type of filter. 

As shown in the table, PPGN fails to learn band-pass filters. This result which contradicts the previous theoretical result is related to memory and complexity issues. Hence, as explained before, PPGN needs a deeper and wider architecture for this task that can not be reached for 900 node graphs (\cite{breakingthelimits}).  


 
 

\begin{table}[h]
    \centering
    \caption{$R^2$ score on spectral filtering node regression. Results are a median of 10 runs.}
    \footnotesize{
    \begin{filecontents*}{data/filteringGMN.txt}
Method,Low-pass,High-pass,Band-pass
CHEBNET,{0.9995},0.9901,{0.8217}
GNNML3,{0.9995},{0.9909},{0.8189}
PPGN,0.9991,\textbf{0.9925},0.1041
\textcolor{blue}{G$^2$N$^2$},\textbf{0.9996},\textbf{0.9994},{0.8206}
\end{filecontents*}
\csvautobooktabular{data/filteringGMN.txt}
    }
    \label{tab:filter}
\end{table}

\section{Conclusion}\label{sec: 5}



Designing provably expressive GNNs has been the target of many recent works. In this paper, we have proposed a new theoretical framework for designing such models. Taking as input a language fragment, i.e. a set of algebraic operations, the framework uses reduced Context Free Grammars to drive the generation of graph neural architectures with provable expressive power. The framework provides insights about the importance of algebraic operations in the resulting model, as shown by the experimental grammar ablative study. Such results can be useful for improving the performance vs. computational cost trade-off for a given task.   

Through the application of the framework to $\ML{\cur L_3}$ fragment, the paper also proposed the provably $\WL 3$ $\ggnn$ model. In addition to these theoretical guarantees, $\ggnn$ is also shown to be efficient for solving graph learning downstream tasks through several experiments on regression, classification and spectral filtering benchmarks. In all cases, $\ggnn$ outperforms $\WL 3$ GNNs, while being more tractable.

Beyond these results, we are convinced that our contributions open the door to the design of models surpassing $\WL 3$, taking as root a language manipulating tensors of greater order (\cite{geerts2022expressiveness}). Moreover, the framework is not limited to GNN models since many other learning paradigm can be modeled with algebraic languages.

\newpage

\bibliography{{./Biblio/biblio}}

\begin{thebibliography}{10}

\bibitem{wesfleman}
AA~Lehman and Boris Weisfeiler.
\newblock A reduction of a graph to a canonical form and an algebra arising
  during this reduction.
\newblock {\em Nauchno-Technicheskaya Informatsiya}, 2(9):12--16, 1968.

\bibitem{morris2019wl}
Christopher Morris, Martin Ritzert, Matthias Fey, William~L Hamilton, Jan~Eric
  Lenssen, Gaurav Rattan, and Martin Grohe.
\newblock Weisfeiler and leman go neural: Higher-order graph neural networks.
\newblock In {\em Proceedings of the AAAI conference on artificial
  intelligence}, volume~33, pages 4602--4609, 2019.

\bibitem{bodnar2021weisfeiler}
Cristian Bodnar, Fabrizio Frasca, Yuguang Wang, Nina Otter, Guido~F Montufar,
  Pietro Lio, and Michael Bronstein.
\newblock Weisfeiler and lehman go topological: Message passing simplicial
  networks.
\newblock In {\em International Conference on Machine Learning}, pages
  1026--1037. PMLR, 2021.

\bibitem{bodnar2021weisfeiler2}
Cristian Bodnar, Fabrizio Frasca, Nina Otter, Yuguang Wang, Pietro Lio, Guido~F
  Montufar, and Michael Bronstein.
\newblock Weisfeiler and lehman go cellular: Cw networks.
\newblock {\em Advances in Neural Information Processing Systems},
  34:2625--2640, 2021.

\bibitem{Zhang2023ACE}
Bohang Zhang, Guhao Feng, Yiheng Du, Di~He, and Liwei Wang.
\newblock A complete expressiveness hierarchy for subgraph gnns via subgraph
  weisfeiler-lehman tests.
\newblock In {\em International Conference on Machine Learning}, 2023.

\bibitem{gilmer2017neural}
Justin Gilmer, Samuel~S Schoenholz, Patrick~F Riley, Oriol Vinyals, and
  George~E Dahl.
\newblock Neural message passing for quantum chemistry.
\newblock In {\em International conference on machine learning}, pages
  1263--1272. PMLR, 2017.

\bibitem{wu2020comprehensive}
Zonghan Wu, Shirui Pan, Fengwen Chen, Guodong Long, Chengqi Zhang, and S~Yu
  Philip.
\newblock A comprehensive survey on graph neural networks.
\newblock {\em IEEE transactions on neural networks and learning systems},
  32(1):4--24, 2020.

\bibitem{xu2018powerful}
Keyulu Xu, Weihua Hu, Jure Leskovec, and Stefanie Jegelka.
\newblock How powerful are graph neural networks?
\newblock In {\em International Conference on Learning Representations}, 2019.

\bibitem{chen2020can}
Zhengdao Chen, Lei Chen, Soledad Villar, and Joan Bruna.
\newblock Can graph neural networks count substructures?
\newblock {\em Advances in neural information processing systems},
  33:10383--10395, 2020.

\bibitem{zhang2021nested}
Muhan Zhang and Pan Li.
\newblock Nested graph neural networks.
\newblock {\em Advances in Neural Information Processing Systems},
  34:15734--15747, 2021.

\bibitem{zhao2021stars}
Lingxiao Zhao, Wei Jin, Leman Akoglu, and Neil Shah.
\newblock From stars to subgraphs: Uplifting any {GNN} with local structure
  awareness.
\newblock In {\em International Conference on Learning Representations}, 2022.

\bibitem{frasca2022understanding}
Fabrizio Frasca, Beatrice Bevilacqua, Michael~M Bronstein, and Haggai Maron.
\newblock Understanding and extending subgraph gnns by rethinking their
  symmetries.
\newblock In {\em Advances in Neural Information Processing Systems}, 2022.

\bibitem{maron2019provably}
Haggai Maron, Heli Ben-Hamu, Hadar Serviansky, and Yaron Lipman.
\newblock Provably powerful graph networks.
\newblock {\em Advances in neural information processing systems}, 32, 2019.

\bibitem{maron2018invariant}
Haggai Maron, Heli Ben-Hamu, Nadav Shamir, and Yaron Lipman.
\newblock Invariant and equivariant graph networks.
\newblock In {\em International Conference on Learning Representations}, 2019.

\bibitem{NEURIPS2020_2f73168b}
Pan Li, Yanbang Wang, Hongwei Wang, and Jure Leskovec.
\newblock Distance encoding: Design provably more powerful neural networks for
  graph representation learning.
\newblock In H.~Larochelle, M.~Ranzato, R.~Hadsell, M.F. Balcan, and H.~Lin,
  editors, {\em Advances in Neural Information Processing Systems}, volume~33,
  pages 4465--4478. Curran Associates, Inc., 2020.

\bibitem{huang2021short}
Ningyuan~Teresa Huang and Soledad Villar.
\newblock A short tutorial on the weisfeiler-lehman test and its variants.
\newblock In {\em ICASSP 2021-2021 IEEE International Conference on Acoustics,
  Speech and Signal Processing (ICASSP)}, pages 8533--8537. IEEE, 2021.

\bibitem{zhang2023expressive}
Bingxu Zhang, Changjun Fan, Shixuan Liu, Kuihua Huang, Xiang Zhao, Jincai
  Huang, and Zhong Liu.
\newblock The expressive power of graph neural networks: A survey.
\newblock {\em arXiv preprint arXiv:2308.08235}, 2023.

\bibitem{Geerts}
FlorisF Geerts.
\newblock On the expressive power of linear algebra on graphs.
\newblock {\em Theory of Computing Systems}, Oct 2020.

\bibitem{matlang}
Robert Brijder, Floris Geerts, Jan~Van den Bussche, and Timmy Weerwag.
\newblock On the expressive power of query languages for matrices.
\newblock {\em {ACM} Trans. Database Syst.}, 44(4):15:1--15:31, 2019.

\bibitem{balcilarexpresspower}
Muhammet Balcilar, Guillaume Renton, Pierre H{\'e}roux, Benoit Ga{\"u}z{\`e}re,
  S{\'e}bastien Adam, and Paul Honeine.
\newblock Analyzing the expressive power of graph neural networks in a spectral
  perspective.
\newblock In {\em International Conference on Learning Representations}, 2020.

\bibitem{breakingthelimits}
Muhammet Balcilar, Pierre H{\'e}roux, Benoit Gauzere, Pascal Vasseur,
  S{\'e}bastien Adam, and Paul Honeine.
\newblock Breaking the limits of message passing graph neural networks.
\newblock In {\em International Conference on Machine Learning}, pages
  599--608. PMLR, 2021.

\bibitem{ramakrishnan2014quantum}
Raghunathan Ramakrishnan, Pavlo~O Dral, Matthias Rupp, and O~Anatole
  Von~Lilienfeld.
\newblock Quantum chemistry structures and properties of 134 kilo molecules.
\newblock {\em Scientific data}, 1(1):1--7, 2014.

\bibitem{wu2018moleculenet}
Zhenqin Wu, Bharath Ramsundar, Evan~N Feinberg, Joseph Gomes, Caleb Geniesse,
  Aneesh~S Pappu, Karl Leswing, and Vijay Pande.
\newblock Moleculenet: a benchmark for molecular machine learning.
\newblock {\em Chemical science}, 9(2):513--530, 2018.

\bibitem{huang2022boosting}
Yinan Huang, Xingang Peng, Jianzhu Ma, and Muhan Zhang.
\newblock Boosting the cycle counting power of graph neural networks with
  {I$^2$}-{GNN}s.
\newblock In {\em The Eleventh International Conference on Learning
  Representations}, 2023.

\bibitem{Morris+2020}
Christopher Morris, Nils~M. Kriege, Franka Bause, Kristian Kersting, Petra
  Mutzel, and Marion Neumann.
\newblock Tudataset: A collection of benchmark datasets for learning with
  graphs.
\newblock In {\em ICML 2020 Workshop on Graph Representation Learning and
  Beyond (GRL+ 2020)}, 2020.

\bibitem{bouritsas2022improving}
Giorgos Bouritsas, Fabrizio Frasca, Stefanos Zafeiriou, and Michael~M
  Bronstein.
\newblock Improving graph neural network expressivity via subgraph isomorphism
  counting.
\newblock {\em IEEE Transactions on Pattern Analysis and Machine Intelligence},
  45(1):657--668, 2022.

\bibitem{chebnet}
David~K. Hammond, Pierre Vandergheynst, and Rémi Gribonval.
\newblock Wavelets on graphs via spectral graph theory.
\newblock {\em Applied and Computational Harmonic Analysis}, 30(2):129--150,
  2011.

\bibitem{geerts2022expressiveness}
Floris Geerts and Juan~L Reutter.
\newblock Expressiveness and approximation properties of graph neural networks.
\newblock In {\em International Conference on Learning Representations}, 2022.

\bibitem{kipf2016semi}
Thomas~N Kipf and Max Welling.
\newblock Semi-supervised classification with graph convolutional networks.
\newblock In {\em 5th International Conference on Learning Representations},
  2017.

\bibitem{shervashidze2011weisfeiler}
Nino Shervashidze, Pascal Schweitzer, Erik~Jan Van~Leeuwen, Kurt Mehlhorn, and
  Karsten~M Borgwardt.
\newblock Weisfeiler-lehman graph kernels.
\newblock {\em Journal of Machine Learning Research}, 12(9), 2011.

\bibitem{du2019graph}
Simon~S Du, Kangcheng Hou, Russ~R Salakhutdinov, Barnabas Poczos, Ruosong Wang,
  and Keyulu Xu.
\newblock Graph neural tangent kernel: Fusing graph neural networks with graph
  kernels.
\newblock {\em Advances in neural information processing systems}, 32, 2019.

\bibitem{zhang2018end}
Muhan Zhang, Zhicheng Cui, Marion Neumann, and Yixin Chen.
\newblock An end-to-end deep learning architecture for graph classification.
\newblock In {\em Proceedings of the AAAI conference on artificial
  intelligence}, volume~32, 2018.

\bibitem{de2020natural}
Pim de~Haan, Taco~S Cohen, and Max Welling.
\newblock Natural graph networks.
\newblock {\em Advances in Neural Information Processing Systems},
  33:3636--3646, 2020.

\bibitem{kolouri2020wasserstein}
Soheil Kolouri, Navid Naderializadeh, Gustavo~K Rohde, and Heiko Hoffmann.
\newblock Wasserstein embedding for graph learning.
\newblock {\em arXiv preprint arXiv:2006.09430}, 2020.

\bibitem{cai2021graphnorm}
Tianle Cai, Shengjie Luo, Keyulu Xu, Di~He, Tie-yan Liu, and Liwei Wang.
\newblock Graphnorm: A principled approach to accelerating graph neural network
  training.
\newblock In {\em International Conference on Machine Learning}, pages
  1204--1215. PMLR, 2021.

\end{thebibliography}

\newpage
\appendix

This appendices provide additional content to the main paper $\ggnn$: Weisfeiler and Lehman go grammatical.
\section{CFG and GNN}\label{app:CFG}

\subsection{From $\ML{\cur L_1}$ to $\WL 1$ GNN}\label{subsec:cfgml1}

In this subsection, the reduction framework described in section \ref{sec: 3} is applied to the fragment $\ML {\cur L_1}$ as shown by Figure \ref{fig:CFGtognn1}. 
\begin{figure}[h]
\centering
\includegraphics[width=\textwidth]{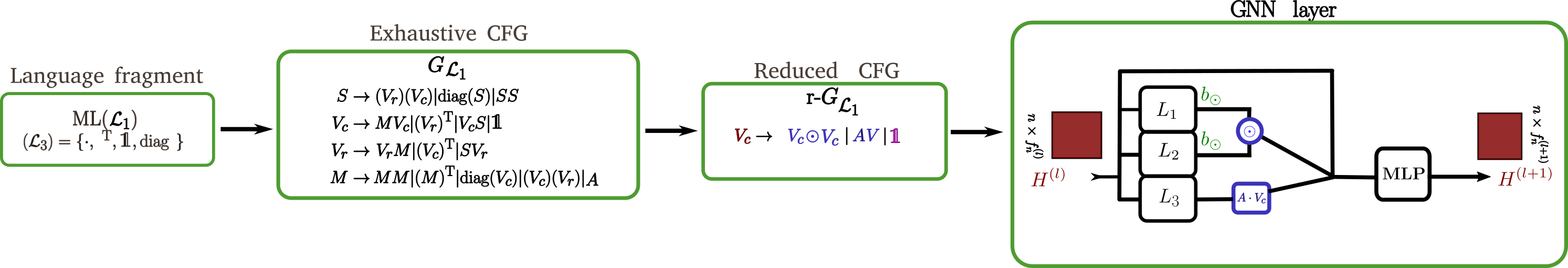}
\caption{Overview of the proposed framework instantiated on $\ML{\cur L_1}$.}
\label{fig:CFGtognn1}
\end{figure}

To determine the variables of the CFG, the following proposition is necessary. 
\begin{prop}\label{prop:statereachML1sup}
For any square matrix of size $n^2$, operations in ${\cur L_1}$ can only produce square matrices of size $n^2$, row or column vectors of size $n$ or scalars. 
\end{prop}

\begin{proof}
Let $M$ be a square matrix of size $n^2$,  we first need to prove that $\cur L_1$ can produce square matrices of size $n^2$, row and column vectors of size $n$ and scalars.

Yet $\onevector := \onevector (M)$ is a column vector of size $n$, $\transpose{\onevector}$ is a row vector of size $n$, $\transpose{\onevector}\cdot \onevector$ is a scalar and $M$ is a square matrix of size $n^2$.

Then let $N \dans \R^{n \times n}$, $v \dans \R^n$, $w \dans \dual{\R^n}$, and $s \dans \R$ be words of $\ML{\cur L_1}$, we have 

\begin{align*}
M\cdot N &\dans \R^{n \times n} &\quad M \cdot v &\dans \R^n &\quad w \cdot M &\dans \dual{\R^n} &\quad w \cdot v &\dans \R \\
v\cdot w &\dans \R^{n \times n} &\quad \onevector(v) &\dans \R^n  &\quad  \transpose v &\dans \dual{\R^n} &\quad \onevector (w) &\dans \R \\
 \transpose M &\dans \R^{n \times n} &\quad \transpose w &\dans \R^n&\quad s \cdot w &\dans \dual{\R^n}  &\quad \diag s &\dans \R \\
 \diag v &\dans\R^{n \times n} &\quad \onevector(M) &\dans \R^n & & & s\cdot s &\dans \R \\
 & &   v\cdot s &\dans \R^n & & & \onevector(s) &\dans \R
\end{align*}
Since this is an exhaustive list of all operations $\ML{\cur L_1}$ can produce with these words, we can conclude.
\end{proof}
\begin{thm}[$\ML{\cur L_1}$ reduced CFG ]\label{thm : ML1reduced}
The following CFG denoted r-$G_{\cur L_1}$ is as expressive as $\WL 1$.
\begin{align}\label{eq:reducedCFG}
V_c &\vers \diag {V_c}V_c \ | \ A V_c \ | \ \onevector
\end{align}
\end{thm}

\begin{proof}
Proposition \ref{prop:statereachML1sup} leads to only four variables. $M$ for the square matrices, $V_c$ for the column vectors, $V_r$ for the row vectors and $S$ for the scalars. We define a CFG $G_{\cur L_1}$ where the rules of a given variable is every possible operation in $\ML{\cur L_1}$ that produce this variable:  
\begin{align}\label{eq:CFGML1}
S &\vers (V_r)(V_c) \ | \ \diag S \ | \ SS   \\
V_c &\vers MV_c \ | \ \transpose{(V_r)} \ | \ V_c S \ | \ \onevector \nonumber \\
V_r &\vers V_r M \ | \ \transpose{(V_c)} \ | \ S V_r \nonumber \\
M &\vers MM \ | \ \transpose{(M)} \ | \ \diag {V_c} \ | \ (V_c)(V_r) \ | \ A  \nonumber 
\end{align}
As any operation in the rules of $G_{\cur L_1}$ belongs to $\cur L_1$, it is clear that $L(G_{\cur L_1}) \inclu \ML {\cur L_1} $. For the other inclusion, a simple inductive proof on the maximal number of rules shows that any sentence produced by $\ML {\cur L_1}$ can be derived from $G_{\cur L_1}$. We have then $\ML {\cur L_1} = L(G_{\cur L_1})$. 
For any scalar $s,s'$, since $\diag s$ and $s\cdot s'$ produce a scalar, the only way to produce a scalar from another variable is to pass through a vector dot product. It implies that to generate scalars, we only need to be able to generate vectors. We can then reduce $G_{\cur L_1}$ by removing the scalar variable $S$ and setting  $V_c$ as starting variable.
\begin{align*}
V_c &\vers MV_c \ | \ \transpose{(V_r)} \ | \ \onevector \\
V_r &\vers V_r M \ | \ \transpose{(V_c)} \\ 
M &\vers MM \ | \ \transpose{(M)} \ | \ \diag {V_c} \ | \ (V_c)(V_r) \ | \ A
\end{align*}
To ensure that the start variable is $V_c$, a mandatory subsequent operation will be $MV_c$ for any matrix variable $M$. As a consequence, by associativity of the matrix multiplication, $MM$ and $(V_c)(V_r)$ can be removed from the rule of $M$.

\begin{align*}
V_c &\vers MV_c \ | \ \transpose{(V_r)} \ | \ \onevector \\
V_r &\vers V_r M \ | \ \transpose{(V_c)}  \\
M &\vers \transpose{(M)} \ | \ \diag {V_c} \ | \ A 
\end{align*}
 Since $\mathrm{diag}$ produces symmetric matrices and $A$ is symmetric, $\transpose{(M)}$ does not play any role here. As a consequence, we can then focus on the column vector and we obtain r-$G_{\cur L_1}$.
\end{proof}

Figure \ref{fig:fullgraph} shows how the CFG $G_{\cur L_1}$ produces the sentence $\transpose\onevector A\onevector$.
 
\begin{figure}[h]
\centering
\includegraphics[scale=.3]{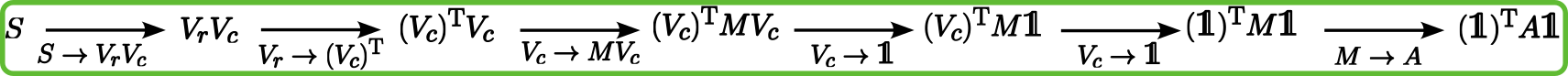}
\caption{$G_{\cur L_1}$ generating the sentence $\transpose\onevector A\onevector$. From the starting variable, Variables are replaced by applying rules until only terminal symbols remain.}
\label{fig:fullgraph}
\end{figure}

\subsection{Proofs of section 3}\label{subsec:proofl3}

This subsection is dedicated to proof of propositions and theorem of section 3.

Firstly, we prove proposition \ref{prop:ML3state}.

\begin{proof}
Since $\cur L_1 \inclu \cur L_3$, we only need to check the rule associated with the matrix Hadamard product can produce. Let $M \dans \R^{n \times n}$ and $N \dans \R^{n \times n}$ be words $\ML{\cur L_3}$ can produce, we have $M \odot N \dans \R^{n \times n}$. We can conclude.
\end{proof}

Secondly, we prove proposition \ref{prop:vectorodot}.

\begin{proof}
Let $M$ be a square matrix, $V_c,V_r$ be respectively column and row vectors, we have for any $i,j$,
\begin{align*}
\left( M \odot (V_c\cdot V_r) \right)_{i,j}  &=  M_{i,j} (V_c \cdot V_r)_{i,j} \\
 &= (V_c)_i M_{i,j} (V_r)_j \\
 &= \sum_l \diag {V_c} _{i,l} M_{l,j}(V_r)_j \\
 &= (\diag {V_c} M)_{i,j} (V_r)_j \\
 &=\sum_l (\diag {V_c} M)_{i,l} \diag {V_r}_{l,j} \\
 &= (\diag {V_c} M \diag {V_r})_{i,j}
\end{align*}
We only use the scalar product commutativity here.
\end{proof}

Eventually, we prove theorem \ref{thm: exhaustivegl3}.

\begin{proof}
As any operation in the rules of $G_{\cur L_3}$ belongs to $\cur L_3$, it is clear that $L(G_{\cur L_3}) \inclu \ML {\cur L_3} $.

Let $k$ be a positive integer, we denote by $M_{\cur L_3}^k$, $Vc_{\cur L_3}^k$, $Vr_{\cur L_3}^k$ and $S_{\cur L_3}^k$ the set of matrices, column vectors, row vectors and scalars that can be produce with at most $k$ operation in $\cur L_3$ from $A$. We also denote by $M_{G}^k$, $Vc_{G}^k$, $Vr_{G}^k$ and $S_{G}^k$ the set of matrices, column vectors, row vectors and scalars that can be produce with at most $k$ rules applied in $G_{\cur L_3}$.

For $k = 0$, we have $Vc_{\cur L_3}^0= Vr_{\cur L_3}^0 = S_{\cur L_3}^0 = \emptyset$, and thus $Vc_{\cur L_3}^0 \inclu Vc_{G}^0$, $Vr_{\cur L_3}^0\inclu Vr_{G}^0$ and $S_{\cur L_3}^0\inclu S_{G}^0$. Moreover $M_{\cur L_3}^0=\{A\}$ and $M_{G}^0=\{A\}$. 

Let suppose that there exists $k\geqslant 0$ such that $M_{\cur L_3}^k \inclu M_{G}^k$, $Vc_{\cur L_3}^k \inclu Vc_{G}^k$, $Vr_{\cur L_3}^k\inclu Vr_{G}^k$ and $S_{\cur L_3}^k\inclu S_{G}^k$. Then since $G_{\cur L_3}$ rules is composed of the exhaustive operations in $\cur L_3$, we have that $M_{\cur L_3}^{k+1} \inclu M_{G}^{k+1}$, $Vc_{\cur L_3}^{k+1} \inclu Vc_{G}^{k+1}$, $Vr_{\cur L_3}^{k+1}\inclu Vr_{G}^{k+1}$ and $S_{\cur L_3}^{k+1}\inclu S_{G}^{k+1}$
By induction, we have that $L(G_{\cur L_3}) \inclu \ML{\cur L_3}$ and we can conclude that $L(G_{\cur L_3}) =\ML{\cur L_3}$.
\end{proof}

\subsection{CFG to describe existent architectures}\label{subsec:cfgarch}

The following examples show how CFG can be used to characterise GNNs.
\begin{prop}[GCN CFG]\label{prop : gcncfg}
The following CFG, strictly less expressive than $\ML{\cur L_1}$, represents GCN (\cite{kipf2016semi})
\begin{align}\label{eq : gcncfg}
V_c &\vers   C V_c \ | \ \onevector
\end{align}
where $C = \diag{(A+I)\onevector}^{-\frac 1 2}(A+I)\diag{(A+I)\onevector}^{-\frac 1 2}$
\end{prop}
In GCN, the only grammatical operation is the message passing given by $CV_c$ where $C$ is the convolution support. The other parts of the model are linear combinations of vectors and MLP, that correspond to $+$,$\times$, and $f$ in the language. Since $+$,$\times$, and $f$ do not affect the expressive power of the language (\cite{Geerts}), they do not appear in the grammar. Actually, any MPNN based on $k$ convolution support $C_i$ included in $\ML{\cur L_1}$ can be described by the following CFG which is strictly less expressive than $\ML{\cur L_1}$:
\begin{align}\label{eq : gcnmpnn}
V_c &\vers   C_1 V_c \ | \ \cdots \ | \ C_k V_c \ | \ \onevector
\end{align}
GNNML1 is an architecture provably $\WL 1$ equivalent (\cite{breakingthelimits}) with the following node update.
\begin{align}\label{eq:GNNML1}
H^{(l+1)} &= H^{(l)}W^{(l,1)} +AH^{(l)}W^{(l,2)} \\
 &\quad + H^{(l)}W^{(l,3)}\odot H^{(l)}W^{(l,1)}. \nonumber
\end{align}
Where $H^{(l)}$ is the matrix of node embedding at layer $l$ and $W^{(l,i)}$ are learnable weight matrices. For any vector $v,w$, since $\diag v w = v\odot w$, the following CFG that describes GNNML1 is equivalent to r-$G_{\cur L_1}$.
\begin{align}
V_c &\vers V_c \odot V_c \ | \ A V_c \ | \ \onevector
\end{align}

\paragraph{From r-$G_{\cur L_3}$ to MPNNs and PPGN}We have already shown that most MPNNs can be written with operations in r-$G_{\cur L_1}$, since $\cur L_1 \inclu \cur L_3$ it stands for r-$G_{\cur L_3}$. PPGN can also be written with  r-$G_{\cur L_3}$. Indeed, at each layer, PPGN applies the matrix multiplication on matched matrices on the third dimension, an operation included in r-$G_{\cur L_3}$. The node features are stacked on the third dimension as diagonal matrices, the $\mathrm{diag}$ operation is also included in r-$G_{\cur L_3}$. As all operations in PPGN are included, r-$G_{\cur L_3}$ generalises PPGN layer. Actually, the following CFG describes the PPGN layer :
\begin{align}\label{eq:reducedCFGPPGN}
M &\vers  MM \ | \ \diag {\onevector} \ | \ A 
\end{align}

\paragraph{2-IGN CFG}
\begin{figure}
\centering 
\includegraphics[scale=.2]{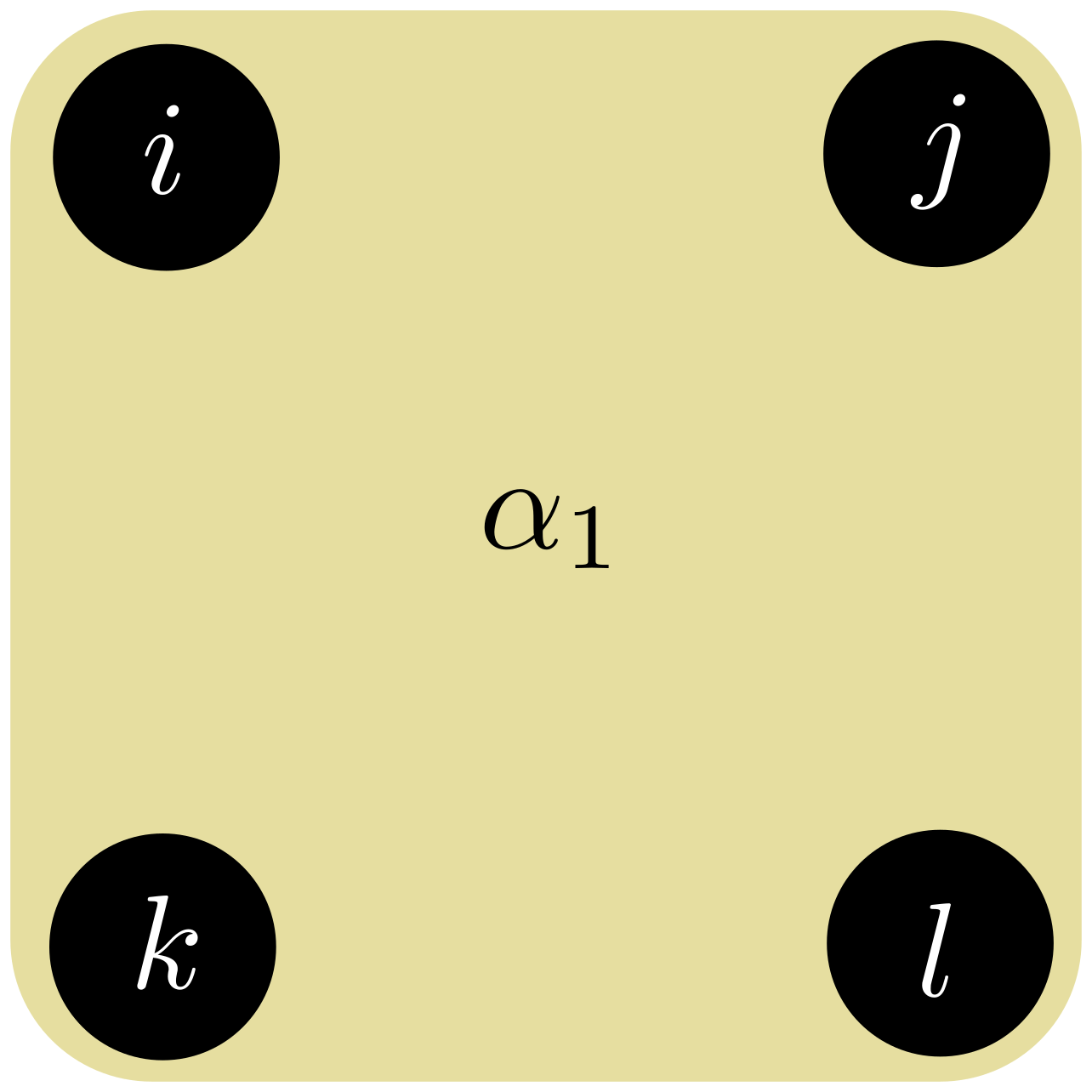}
\includegraphics[scale=.078]{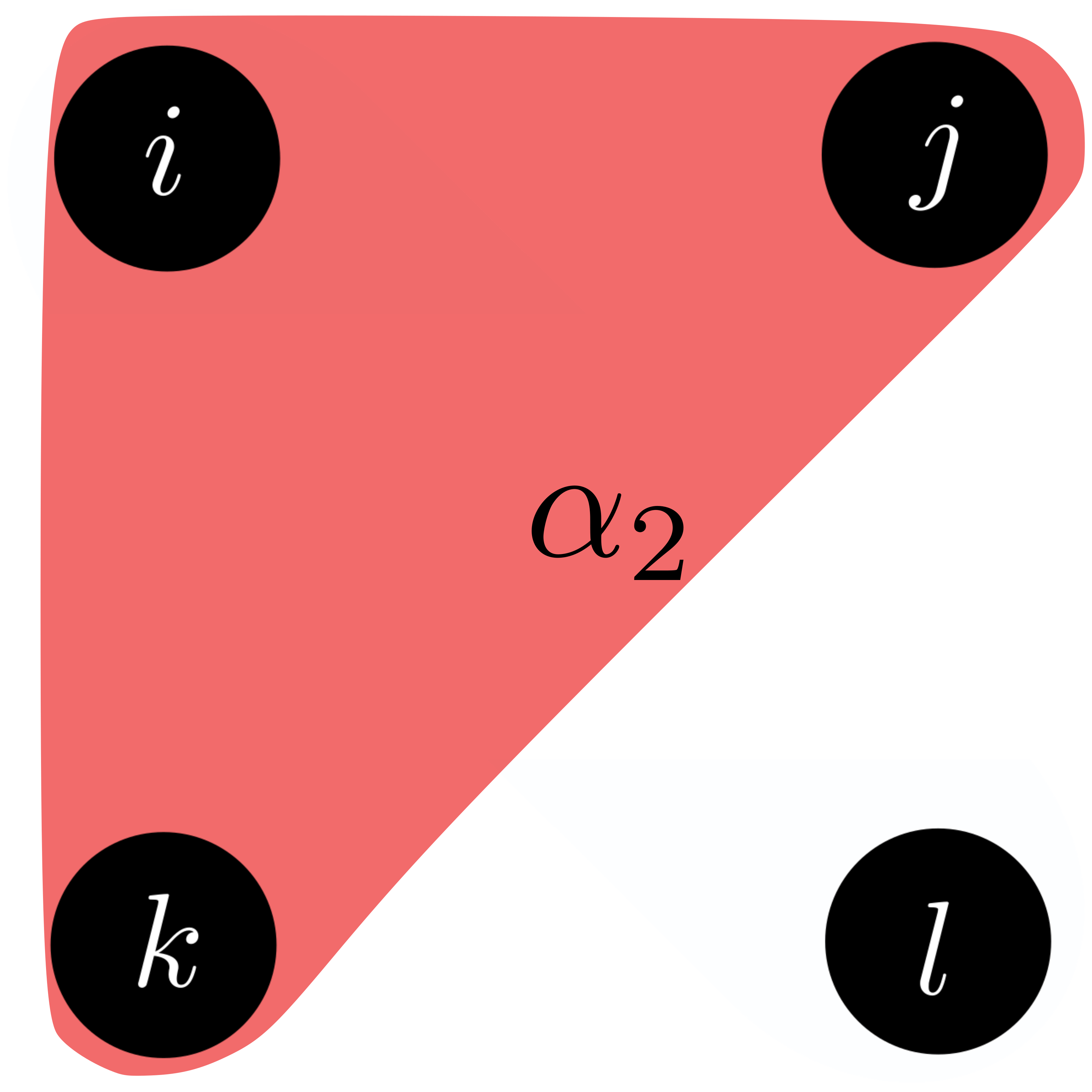}
\includegraphics[scale=.078]{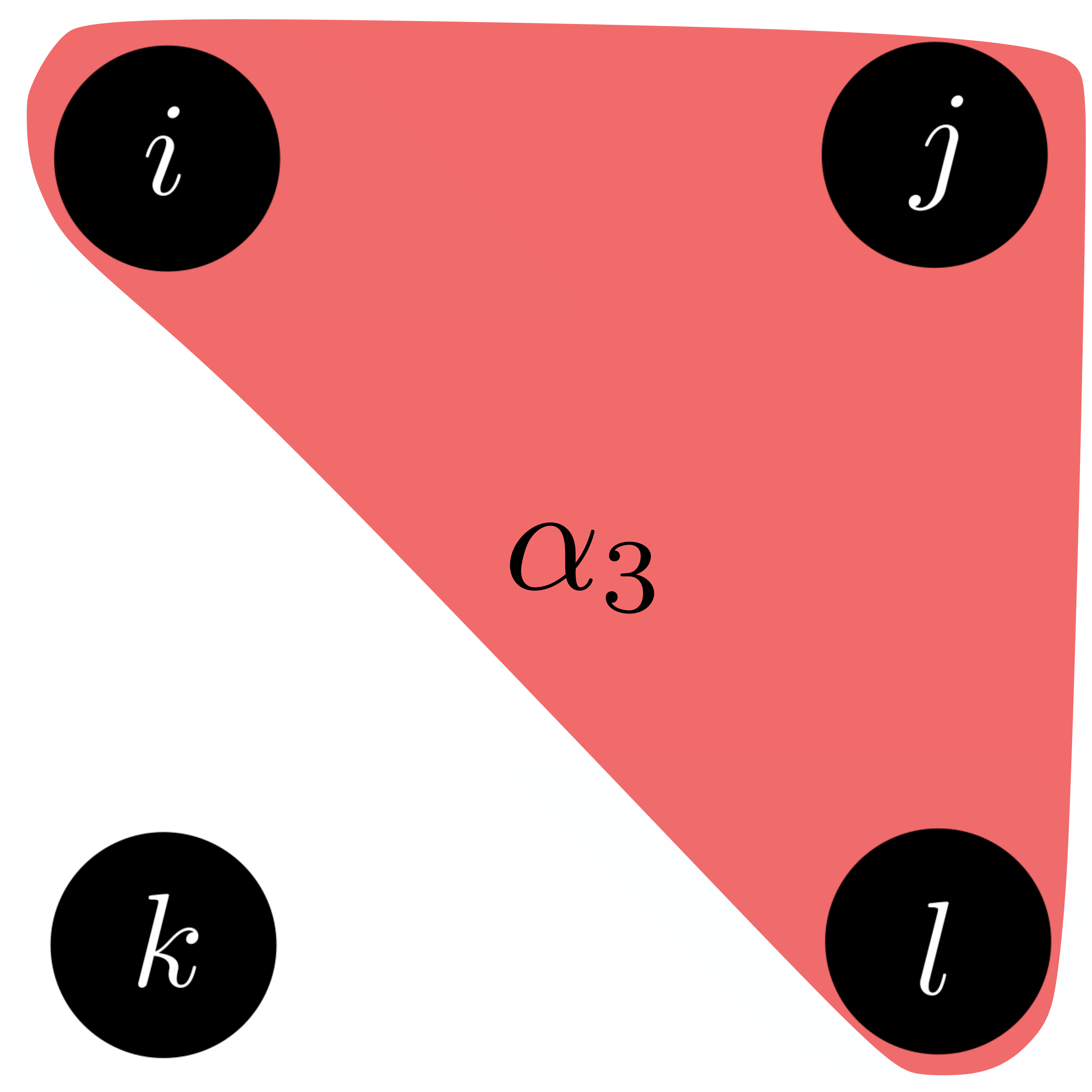}
\includegraphics[scale=.078]{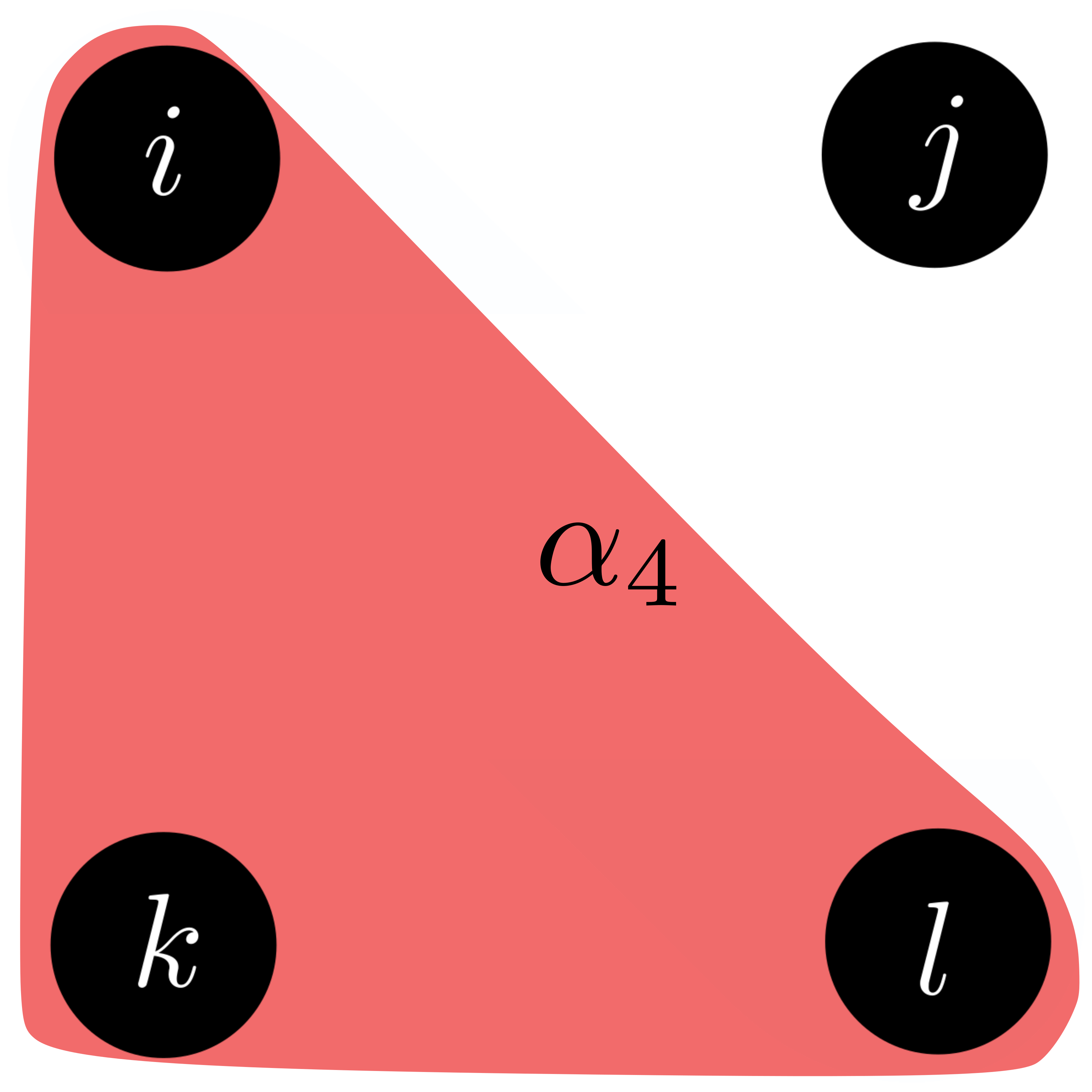}
\includegraphics[scale=.078]{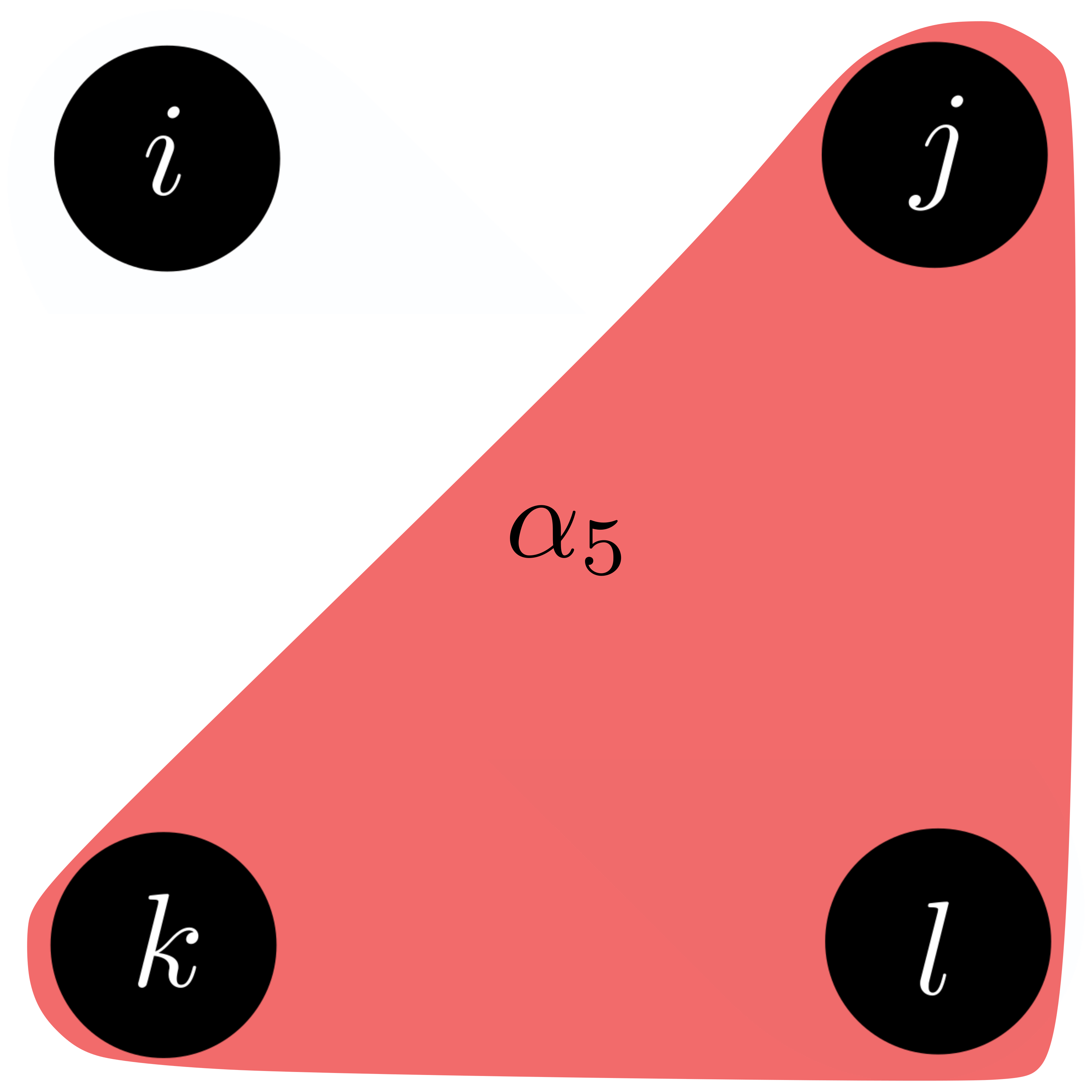}

\includegraphics[scale=.2]{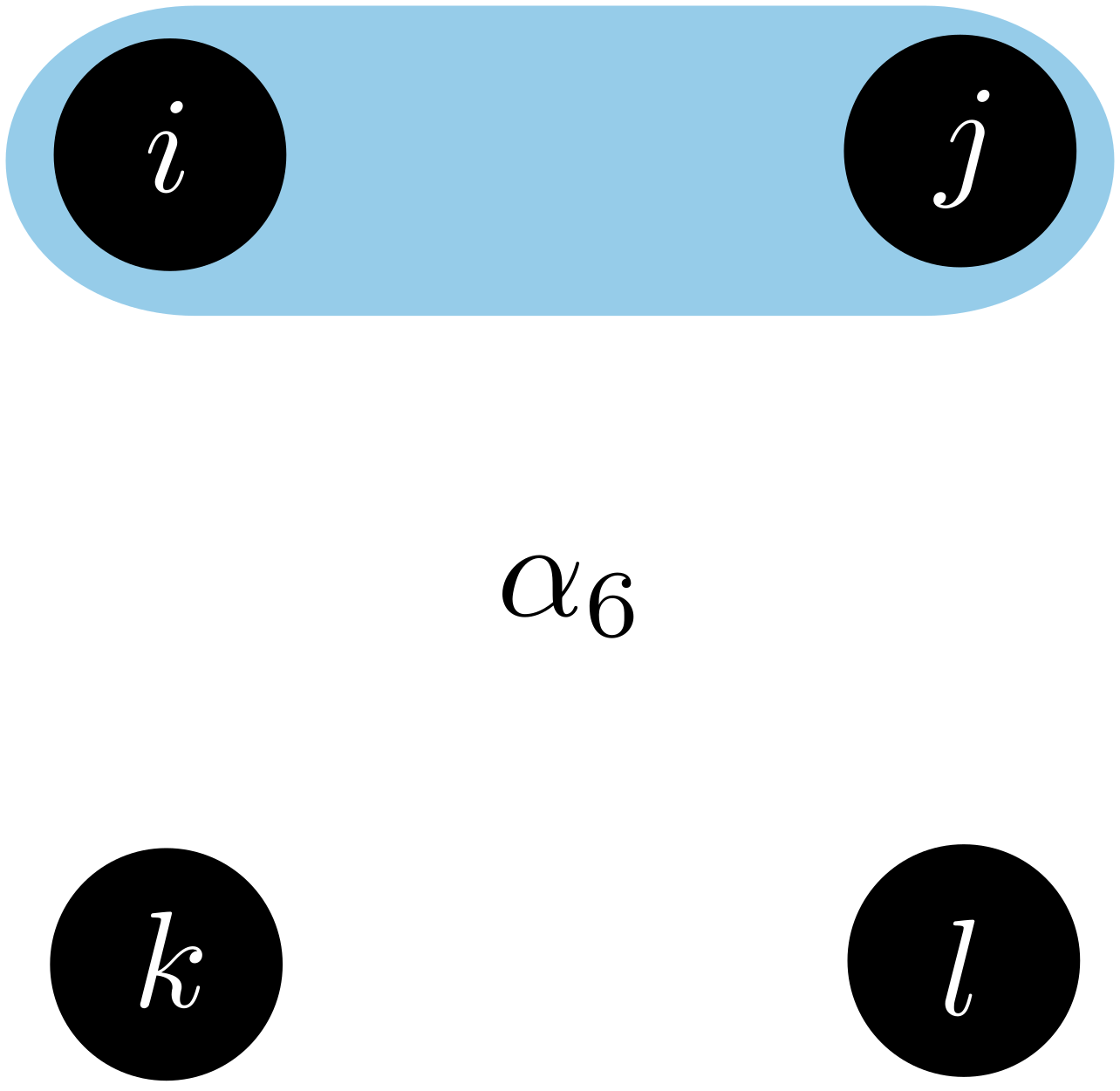}
\includegraphics[scale=.2]{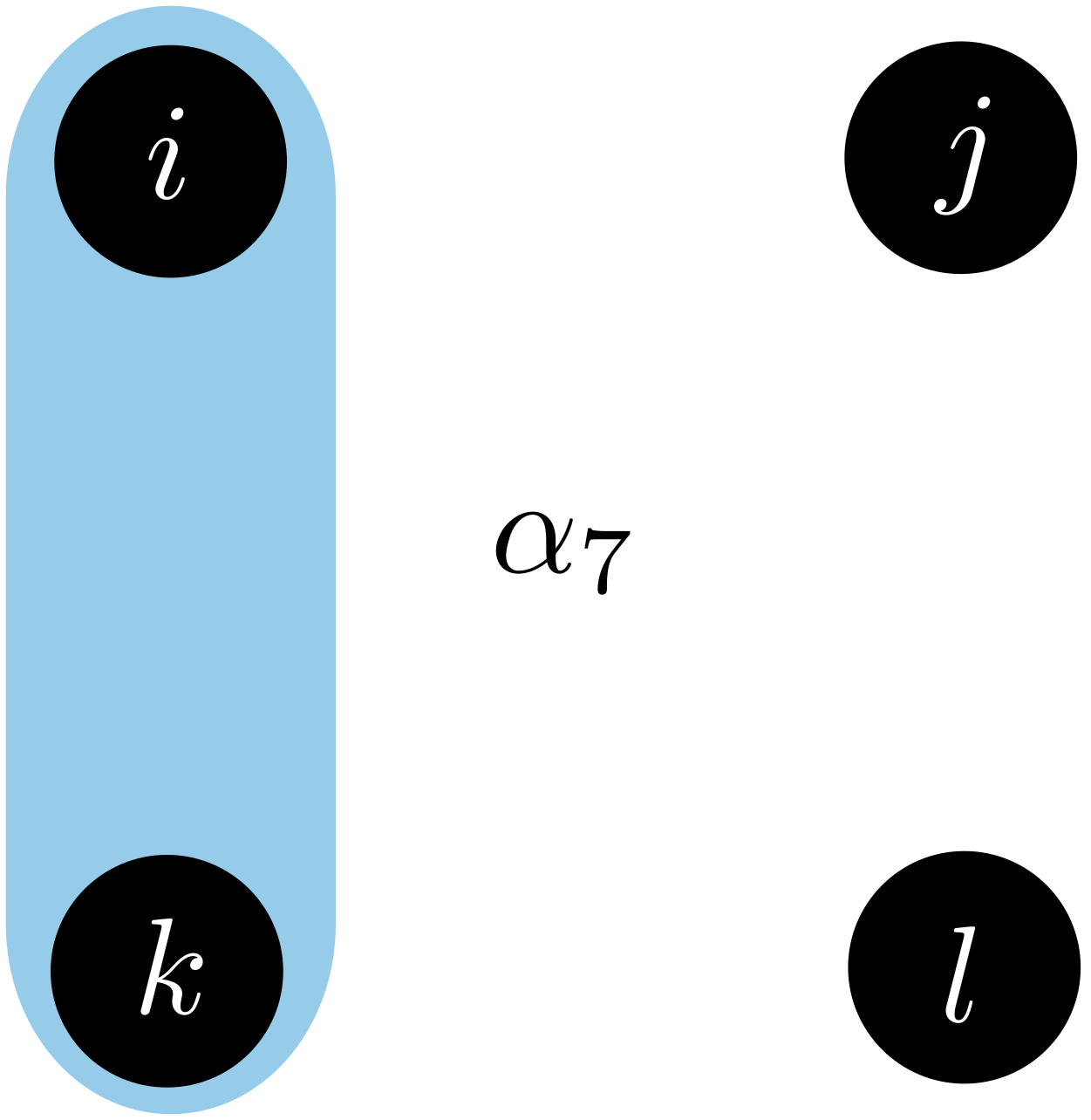}
\includegraphics[scale=.2]{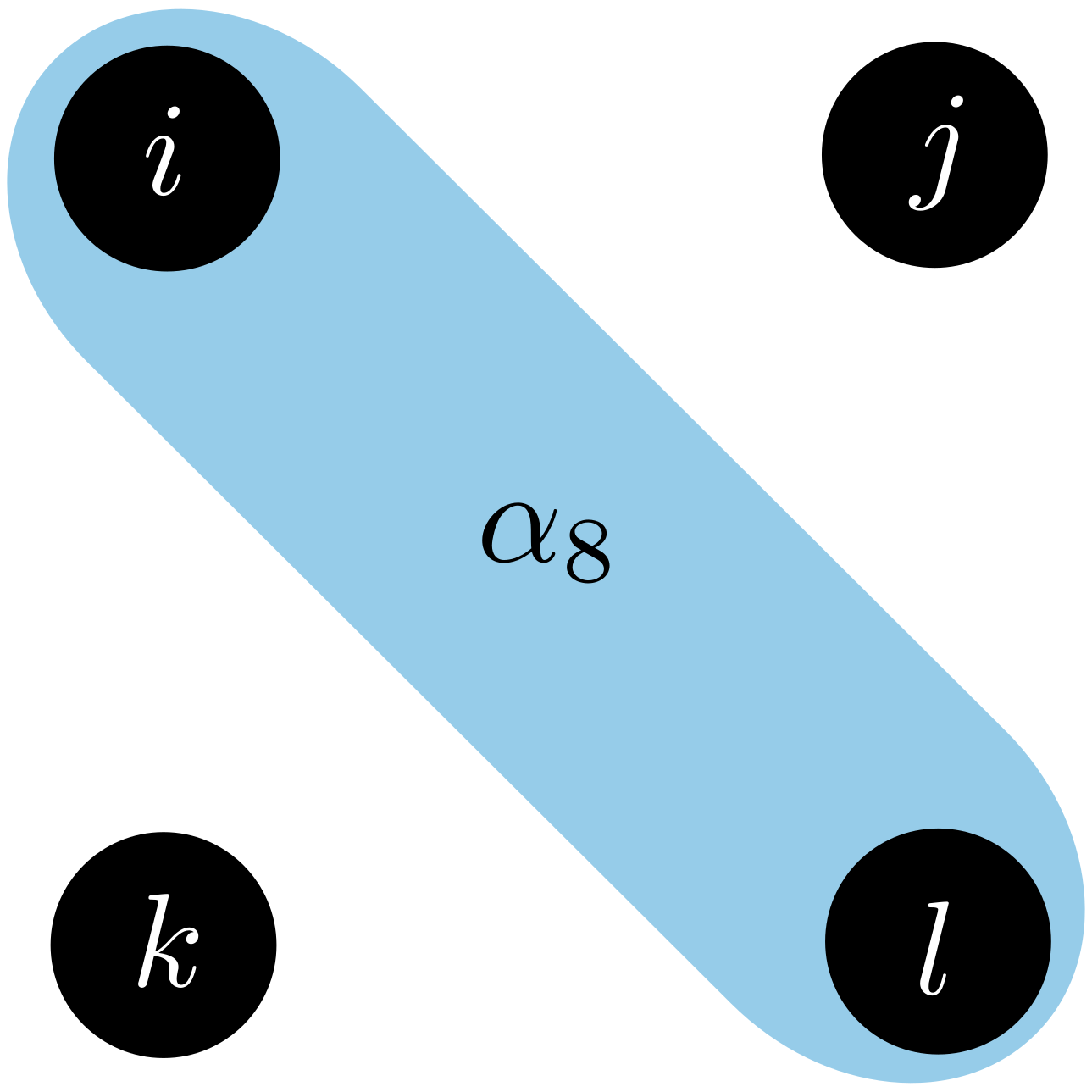}
\includegraphics[scale=.2]{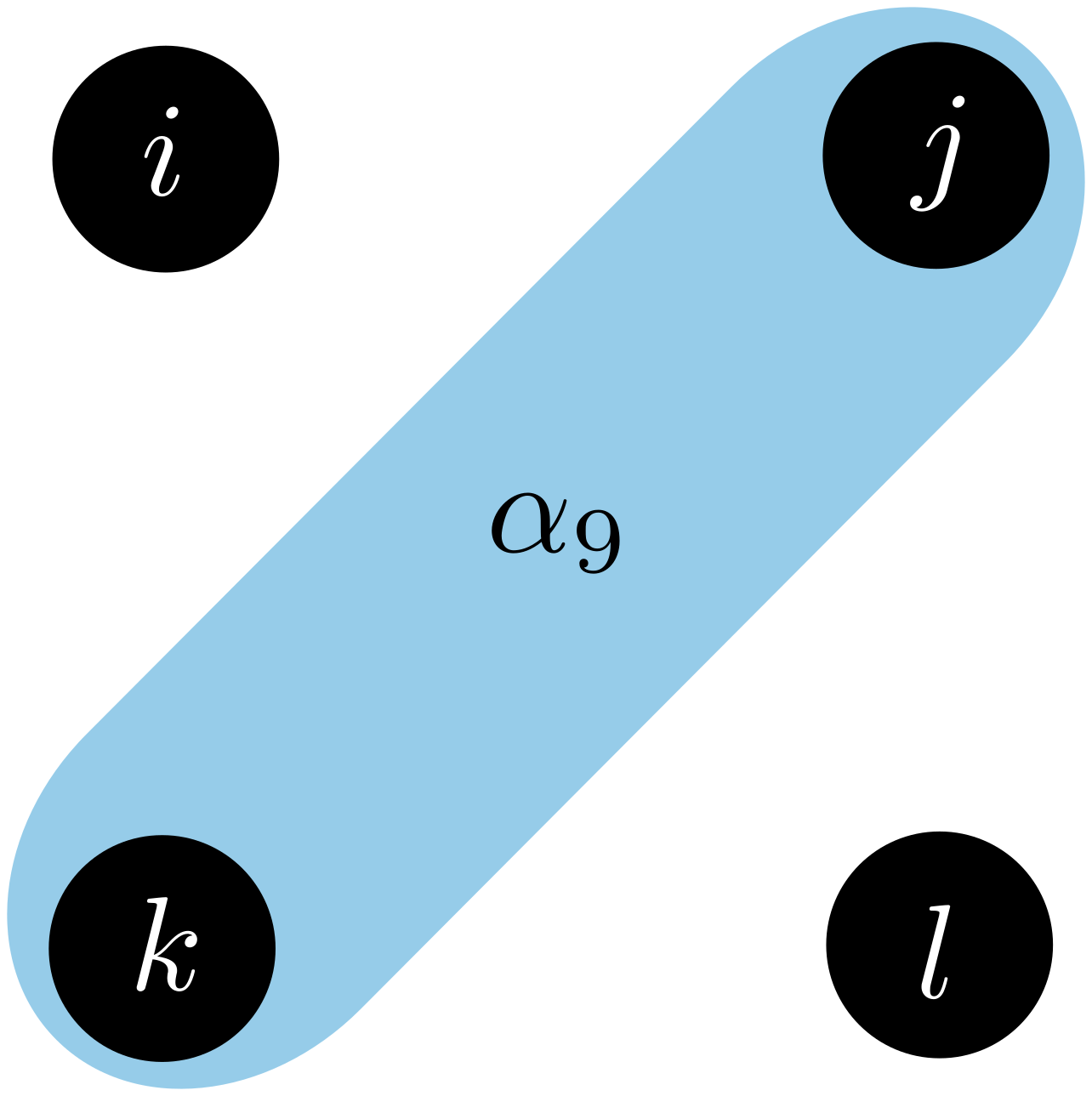}
\includegraphics[scale=.2]{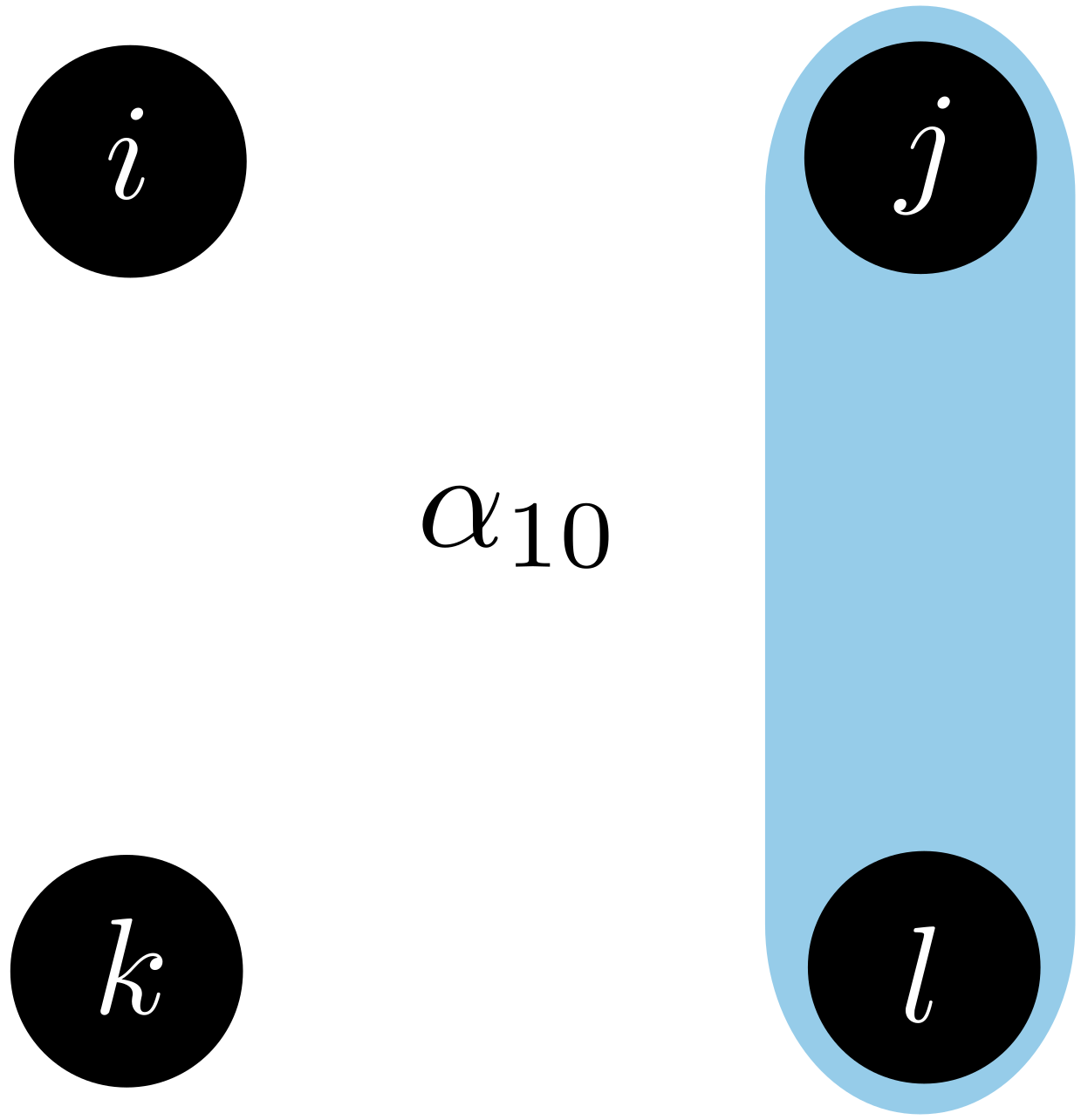}
\includegraphics[scale=.2]{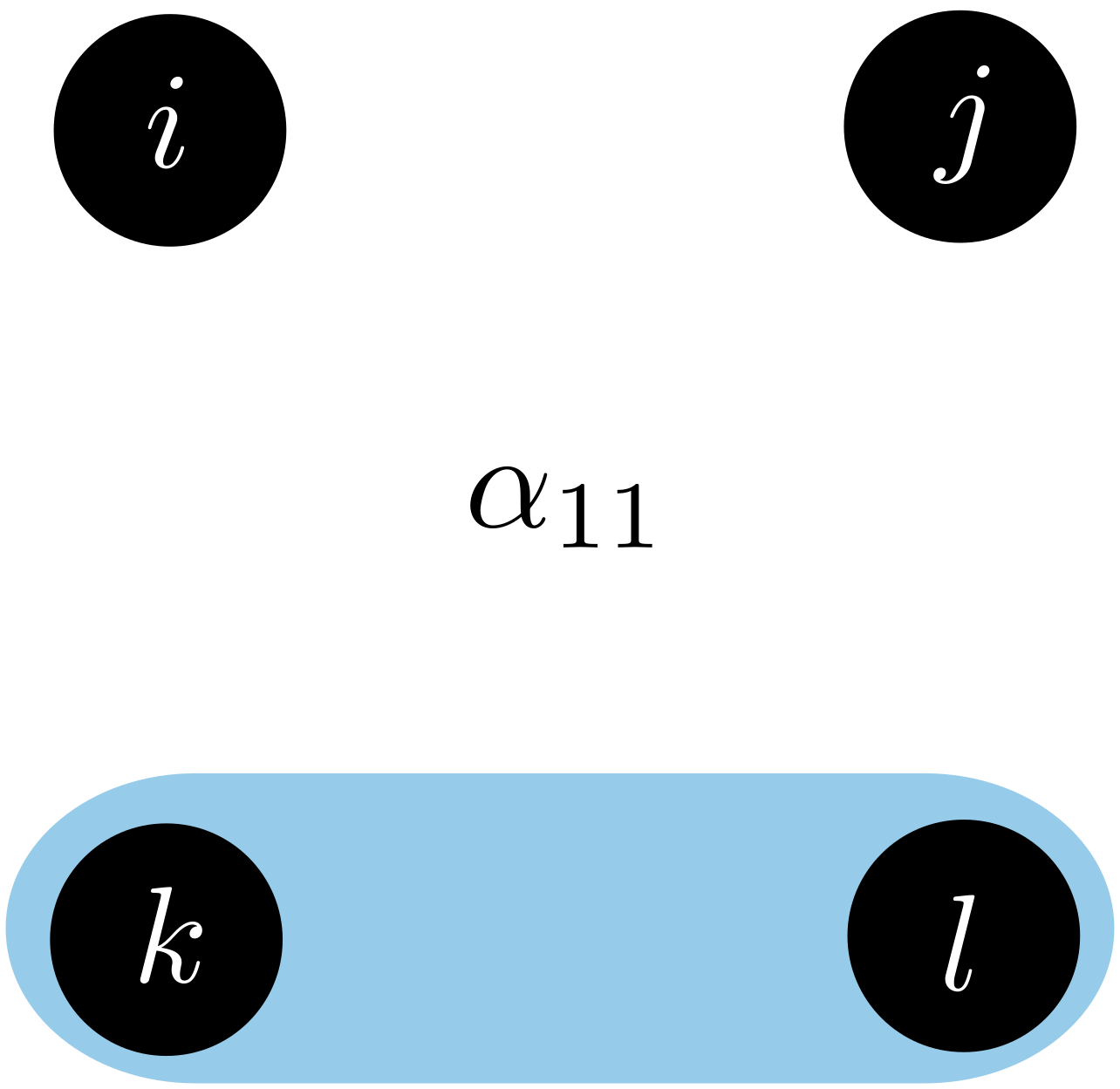}

\includegraphics[scale=.2]{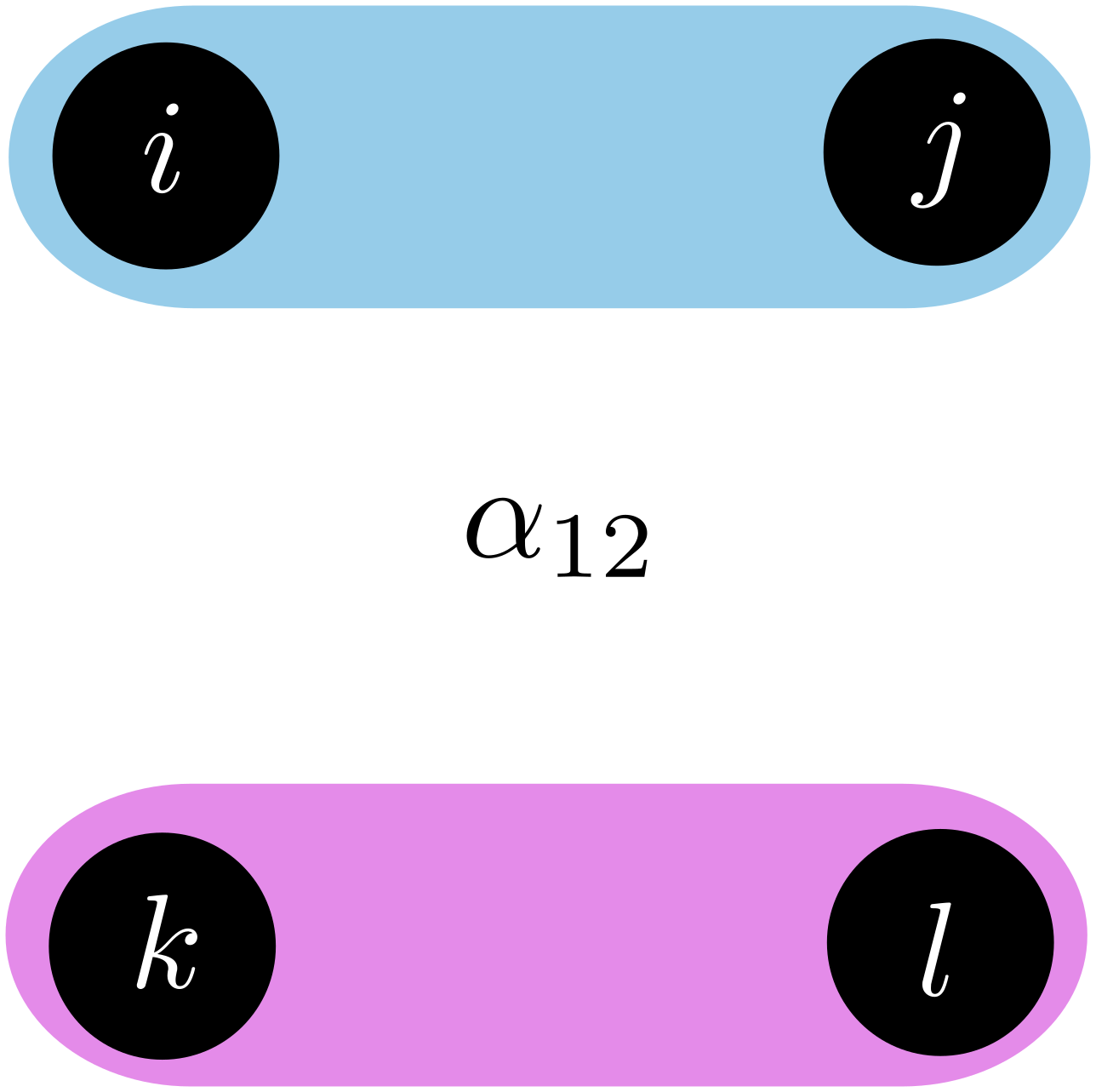}
\includegraphics[scale=.2]{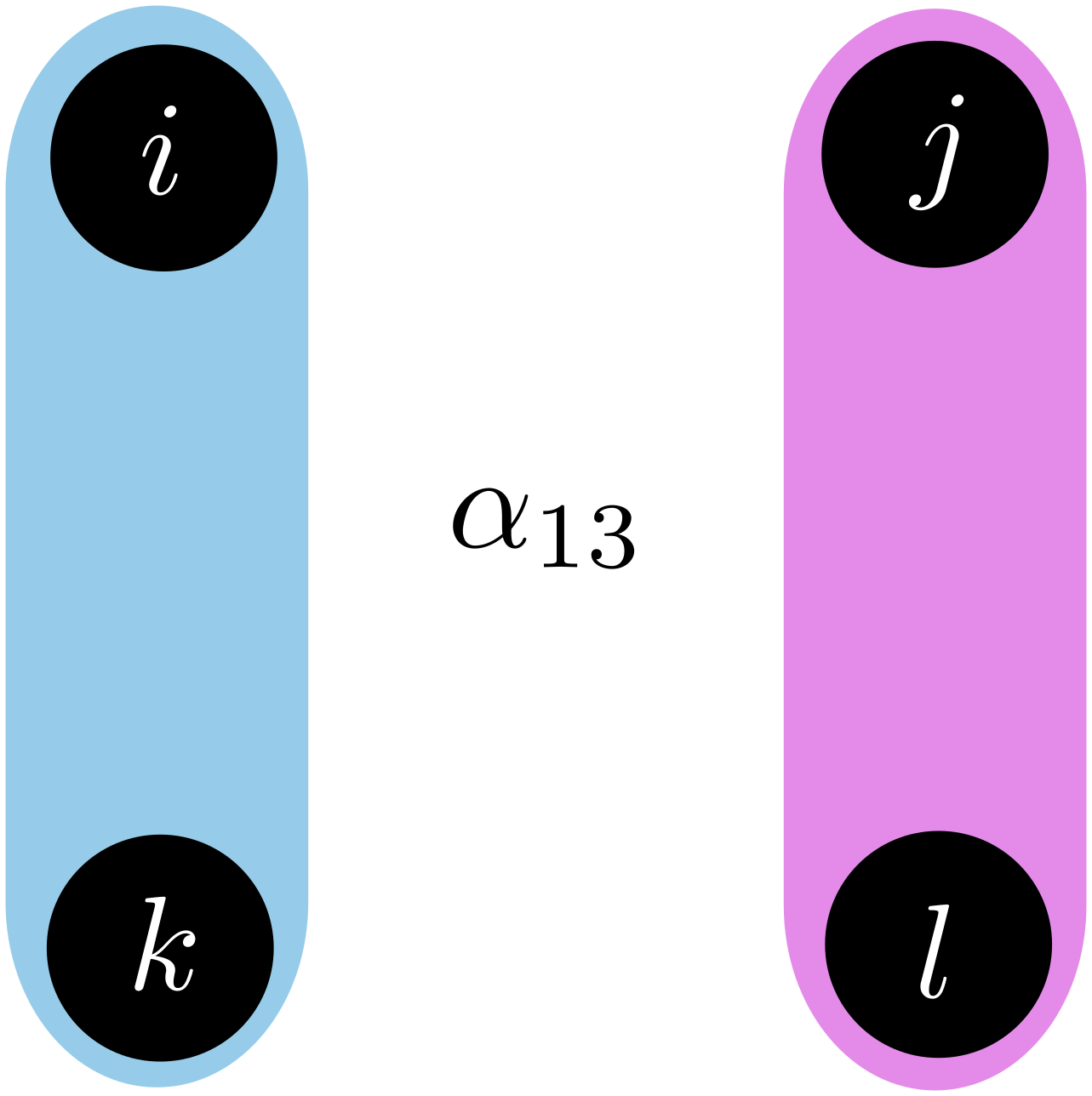}
\includegraphics[scale=.2]{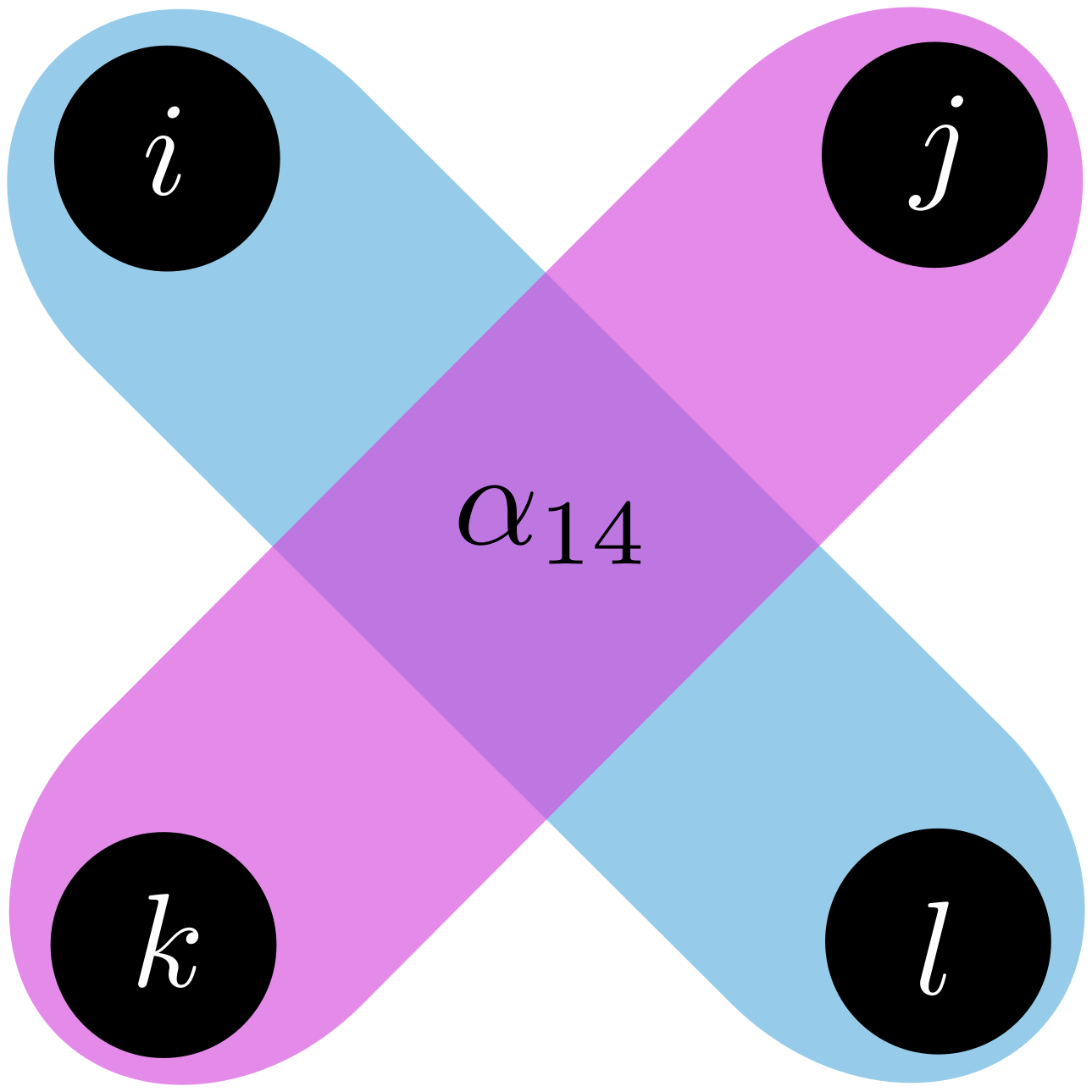}
\includegraphics[scale=.2]{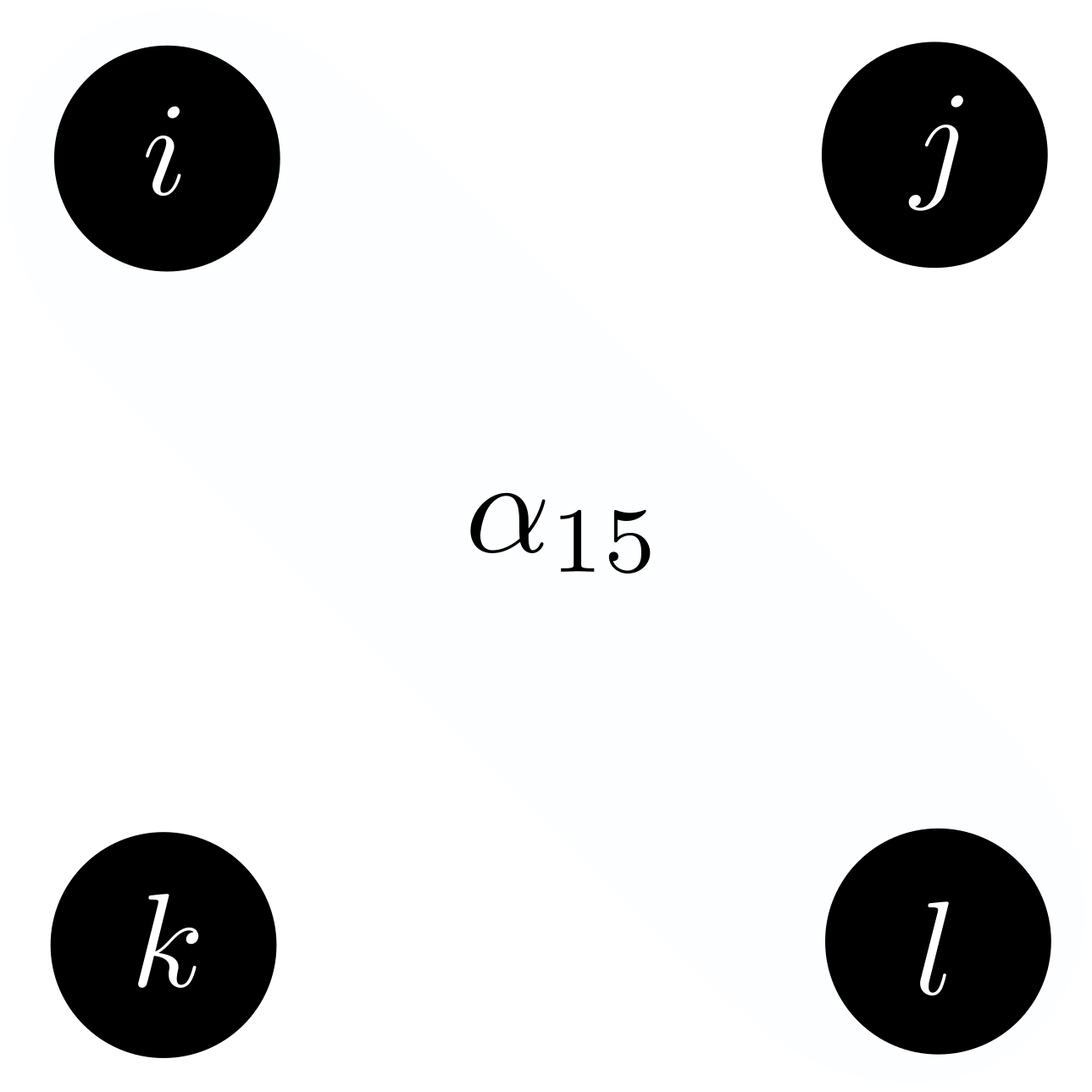}
\caption{Partition of four indices tuples.}
\label{fig : partition4tuple}
\end{figure}
For $p \dans \intervalleentier 1 {15}$, we define $\cur B_p \dans (\R^{n})^4$ as follow, $\cur B_p = \begin{cases} 1 \text{ if }(i,j,k,l) \dans \alpha_p, \\ 0 \text{ if not.} \end{cases} $. Where $(\alpha_p)$ corresponds to the 15 manners to partition four elements that can be seen in Figure \ref{fig : partition4tuple}.

As shown in \cite{maron2018invariant}, $(\cur B_p)$ is a basis of the set of equivariant linear operators from $(\R^{n})^2$ to $(\R^{n})^2$. For the proof in the paper, two isomorphisms $vec : (\R^n)^2 \vers \R^{n^2}$ and $mat : (\R^n)^4 \vers (\R^{n^2})^2 $ was defined for any tensor $T \dans (\R^n)^4$, matrices $ M \dans (\R^{n^2})^2$, $N\dans (\R^n)^2$ and vector $v\dans \R^{n^2}$.
\begin{align*}
mat(T)_{i,j} &= T_{i// n,i \% n,j//n,j \% n} \\
mat^{-1}(M)_{i,j,k,l} &= M_{in + j,kn + l} \\
vec(N)_i &= N_{i//n,i\% n} \\
vec^{-1}(v)_{i,j} &= v_{in + j}
\end{align*}

We can then define the binary operation $\tilde \cdot$ as follow
\begin{align*}
    T \tilde \cdot N = vec^{-1}(mat(T)vec(N))
\end{align*}
Actually, we obtain the following operation
\begin{align*}
    (T\tilde \cdot N)_{i,j}  = \sum_{k,l} T_{i,j,k,l}N_{k,l}
\end{align*}

We have all we need to proceed on writing 2-IGN as a grammar. The idea is to compute the basis operator to any matrices with a set of rules.

\begin{align*}
    (\cur B_1\tilde \cdot N)_{i,j} &= \sum_{k,l} (\cur B_1)_{i,j,k,l}N_{k,l} \\
        &= \begin{cases}
            N_{i,i} \text{ if } i=j,\\
            0 \text{ if not.}
        \end{cases}
\end{align*}
It is pretty easy to see that \begin{align*}
    \cur B_1\tilde \cdot N = N\odot \identite 
\end{align*}

\begin{align*}
    (\cur B_2\tilde \cdot N)_{i,j} &= \sum_{k,l} (\cur B_2)_{i,j,k,l}N_{k,l} \\
        &= \begin{cases}
            \sum_{l \neq i} N_{i,l} \text{ if } i=j,\\
            0 \text{ if not.}
        \end{cases}
\end{align*}
Here, it is a sum over the matrix line avoiding the diagonal. \begin{align*}
    \cur B_2\tilde \cdot N = \diag{(N\odot J)\onevector} 
\end{align*}

\begin{align*}
    (\cur B_3\tilde \cdot N)_{i,j} &= \sum_{k,l} (\cur B_3)_{i,j,k,l}N_{k,l} \\
        &= \begin{cases}
            \sum_{l \neq i} N_{l,i} \text{ if } i=j,\\
            0 \text{ if not.}
        \end{cases}
\end{align*}
Here, it is a sum over the matrix column avoiding the diagonal. \begin{align*}
    \cur B_3\tilde \cdot N = \diag{\transpose{(N\odot J)}\onevector} 
\end{align*}

\begin{align*}
    (\cur B_4\tilde \cdot N)_{i,j} &= \sum_{k,l} (\cur B_4)_{i,j,k,l}N_{k,l} \\
        &= \begin{cases}
             N_{j,j} \text{ if } i\neq j,\\
            0 \text{ if not.}
        \end{cases}
\end{align*}
It is the projection of the corresponding column diagonal element. \begin{align*}
    \cur B_4\tilde \cdot N = (\onevector \transpose \onevector (N\odot \identite))\odot J 
\end{align*}

\begin{align*}
    (\cur B_5\tilde \cdot N)_{i,j} &= \sum_{k,l} (\cur B_5)_{i,j,k,l}N_{k,l} \\
        &= \begin{cases}
             N_{i,i} \text{ if } i\neq j,\\
            0 \text{ if not.}
        \end{cases}
\end{align*}
It is the projection of the corresponding line diagonal element.\begin{align*}
    \cur B_5\tilde \cdot N = ( (N\odot \identite)\onevector \transpose \onevector)\odot J 
\end{align*}

\begin{align*}
    (\cur B_6\tilde \cdot N)_{i,j} &= \sum_{k,l} (\cur B_6)_{i,j,k,l}N_{k,l} \\
        &= \begin{cases}
            \sum_{l \neq k} N_{k,l} - \sum_{l} N_{i,l} -\sum_{l} N_{l,i} \text{ if } i = j,\\
            0 \text{ if not.}
        \end{cases}
\end{align*}
One can recognise $\cur B_2$ and $\cur B_3$. \begin{align*}
    \cur B_6\tilde \cdot N = (\onevector(N\odot J) \transpose \onevector ) \identite - \cur B_2 \tilde \cdot N - \cur B_3 \tilde \cdot N
\end{align*}

\begin{align*}
    (\cur B_7\tilde \cdot N)_{i,j} &= \sum_{k,l} (\cur B_7)_{i,j,k,l}N_{k,l} \\
        &= \begin{cases}
            \sum_{l \neq i} N_{i,l} - N_{i,j} \text{ if } i\neq j,\\
            0 \text{ if not.}
        \end{cases}
\end{align*}
It is just a sum over the line avoiding the element.
\begin{align*}
    \cur B_7\tilde \cdot N = (\onevector \transpose \onevector (N\odot J) - N)\odot J
\end{align*}

\begin{align*}
    (\cur B_8\tilde \cdot N)_{i,j} &= \sum_{k,l} (\cur B_8)_{i,j,k,l}N_{k,l} \\
        &= \begin{cases}
            \sum_{l} N_{l,i} - N_{j,i} \text{ if } i \neq j,\\
            0 \text{ if not.}
        \end{cases}
\end{align*}
It is just a sum over the column corresponding to the line avoiding the transpose element. \begin{align*}
    \cur B_8\tilde \cdot N = ( (N\odot J)\onevector \transpose \onevector - \transpose N)\odot J
\end{align*}

\begin{align*}
    (\cur B_9\tilde \cdot N)_{i,j} &= \sum_{k,l} (\cur B_9)_{i,j,k,l}N_{k,l} \\
        &= \begin{cases}
            \sum_{l \neq i} N_{j,l} - N_{j,i} \text{ if } i\neq j,\\
            0 \text{ if not.}
        \end{cases}
\end{align*}
It is just a sum over the line corresponding to the column avoiding the transpose element.
\begin{align*}
    \cur B_9\tilde \cdot N = (\onevector \transpose \onevector (N\odot J) - \transpose N)\odot J
\end{align*}

\begin{align*}
    (\cur B_{10}\tilde \cdot N)_{i,j} &= \sum_{k,l} (\cur B_{10})_{i,j,k,l}N_{k,l} \\
        &= \begin{cases}
            \sum_{l} N_{l,j} - N_{i,j} \text{ if } i \neq j,\\
            0 \text{ if not.}
        \end{cases}
\end{align*}
It is just a sum over the column avoiding the element. \begin{align*}
    \cur B_{10}\tilde \cdot N = ( (N\odot J)\onevector \transpose \onevector - N)\odot J
\end{align*}

\begin{align*}
    (\cur B_{11}\tilde \cdot N)_{i,j} &= \sum_{k,l} (\cur B_{11})_{i,j,k,l}N_{k,l} \\
        &= \begin{cases}
            \sum_{l} N_{l,l} - N_{i,i} - N_{j,j} \text{ if } i\neq j,\\
            0 \text{ if not.}
        \end{cases}
\end{align*}
It is just a sum over the diagonal avoiding the two corresponding diagonal elements.
\begin{align*}
    \cur B_{11}\tilde \cdot N = (\transpose  \onevector (N\odot \identite )  \onevector) J - \cur B_3\tilde \cdot N - \cur B_4\tilde \cdot N
\end{align*}

\begin{align*}
    (\cur B_{12}\tilde \cdot N)_{i,j} &= \sum_{k,l} (\cur B_{12})_{i,j,k,l}N_{k,l} \\
        &= \begin{cases}
            \sum_{l} N_{l,l} - N_{i,i} \text{ if } i = j,\\
            0 \text{ if not.}
        \end{cases}
\end{align*}
It is just a sum over the diagonal avoiding the corresponding diagonal element. \begin{align*}
    \cur B_{12}\tilde \cdot N = (\transpose  \onevector (N\odot \identite ) \onevector) J - (\onevector \transpose \onevector (N\odot \identite))\odot \identite
\end{align*}

\begin{align*}
    (\cur B_{13}\tilde \cdot N)_{i,j} &= \sum_{k,l} (\cur B_{13})_{i,j,k,l}N_{k,l} \\
        &= \begin{cases}
            N_{i,j} \text{ if } i\neq j,\\
            0 \text{ if not.}
        \end{cases}
\end{align*}
It selects the non-diagonal.
\begin{align*}
    \cur B_{13}\tilde \cdot N = N \odot J
\end{align*}

\begin{align*}
    (\cur B_{14}\tilde \cdot N)_{i,j} &= \sum_{k,l} (\cur B_{14})_{i,j,k,l}N_{k,l} \\
        &= \begin{cases}
            N_{j,i} \text{ if } i \neq j,\\
            0 \text{ if not.}
        \end{cases}
\end{align*}
It selects the transpose non-diagonal. \begin{align*}
    \cur B_{14}\tilde \cdot N = \transpose N \odot J
\end{align*}

\begin{align*}
    (\cur B_{15}\tilde \cdot N)_{i,j} &= \sum_{k,l} (\cur B_{15})_{i,j,k,l}N_{k,l} \\
        &= \begin{cases}
            &\sum_{k\neq l} N_{k,l} - \sum_{i\neq l} N_{i,l}  \\
            &- \sum_{i\neq l} N_{l,i} - \sum_{j\neq l} N_{j,l} \\
            &- \sum_{j\neq l} N_{l,j}- N_{i,j} - N_{j,i} \text{ if } i\neq j,\\
            &0 \text{ if not.}
        \end{cases}
\end{align*}
It is in fact a composition of other elements of the basis.
\begin{align*}
    \cur B_{15}\tilde \cdot N = &(\transpose \onevector(N\odot J)  \onevector ) J - \cur B_{7}\tilde \cdot N - \cur B_{8}\tilde \cdot N \\
    &- \cur B_{9}\tilde \cdot N - \cur B_{10}\tilde \cdot N + \cur B_{13}\tilde \cdot N + \cur B_{14}\tilde \cdot N
\end{align*}

 From all this, we can deduce the following grammar that generates 2-IGN:
 \begin{align*}
     M &\vers V_c \transpose \onevector \ | \ M\odot J \ | \ M \odot \identite \ | \ A \\
     V_c &\vers MV_c \ | \ \onevector 
 \end{align*}

 As one can see, there is less operation in the CFG than operators in the basis.

 \subsection{GNNs derived from different Grammars}

 This subsection is dedicated to a description of GNN derived from different grammar of \textbf{Q1} experiment ( section \ref{sec: 4}. 

Figure \ref{fig:exhaust_gnn} depicts a layer of a GNN derived from the exhaustive CFG $G_{\cur L_3}$. The resulting architecture inherits $\WL 3$ expressive power from theorem \ref{thm: exhaustivegl3}. In Figure \ref{fig:interm_gnn}, a GNN derived from i-$G_{\cur L_3}$, the CFG obtain during the reduction process of the framework of section \ref{sec: 3}, is described. Since $\transpose \ $ is missing in r-$G(\cur L_3$, Figure \ref{fig:gnn_with_transpose} describes a GNN derived from a grammar containing r-$G(\cur L_3)$ and $\transpose M$.
 \begin{figure}
     \centering
     \includegraphics[scale = .14]{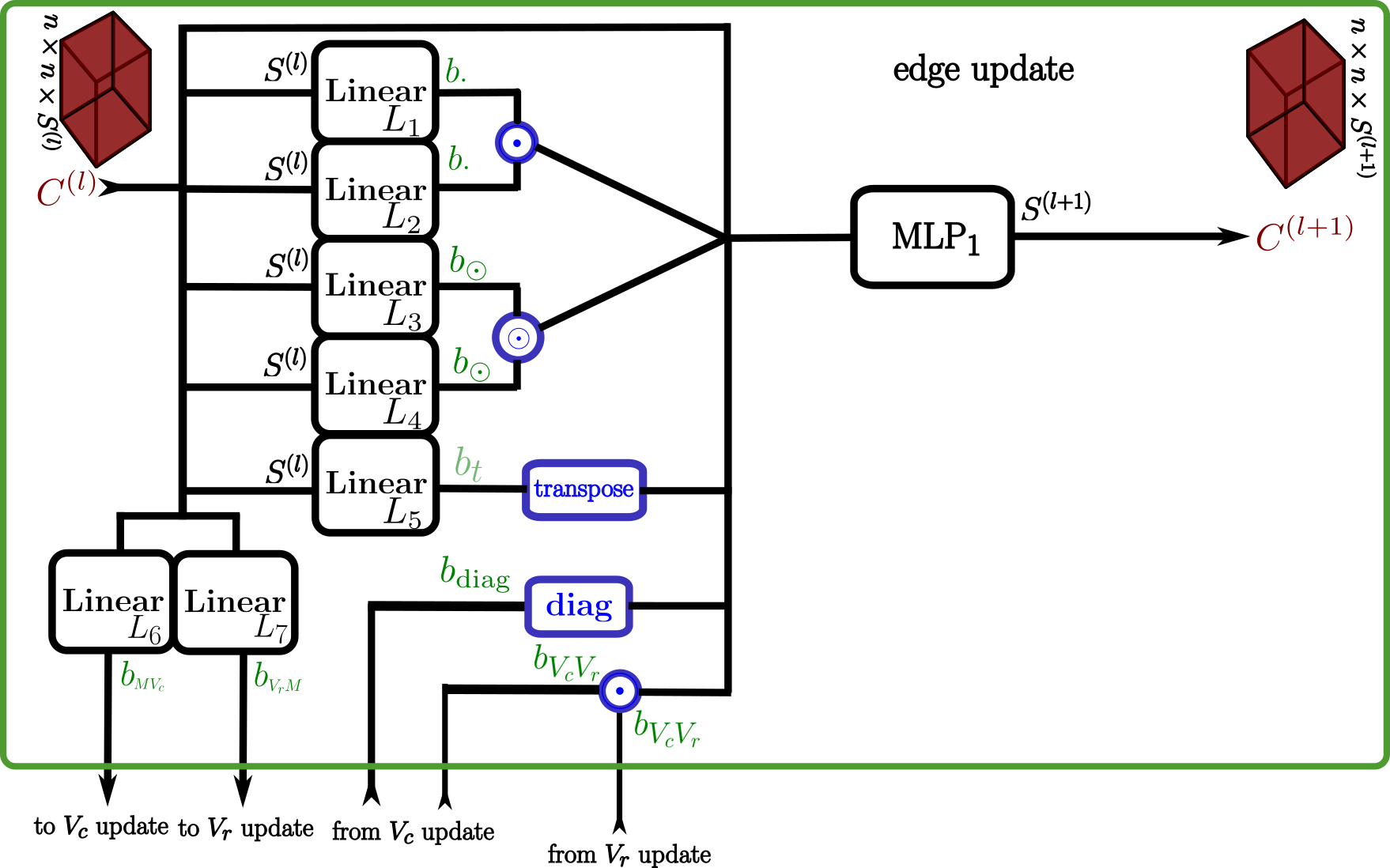}
     
     \includegraphics[scale = .14]{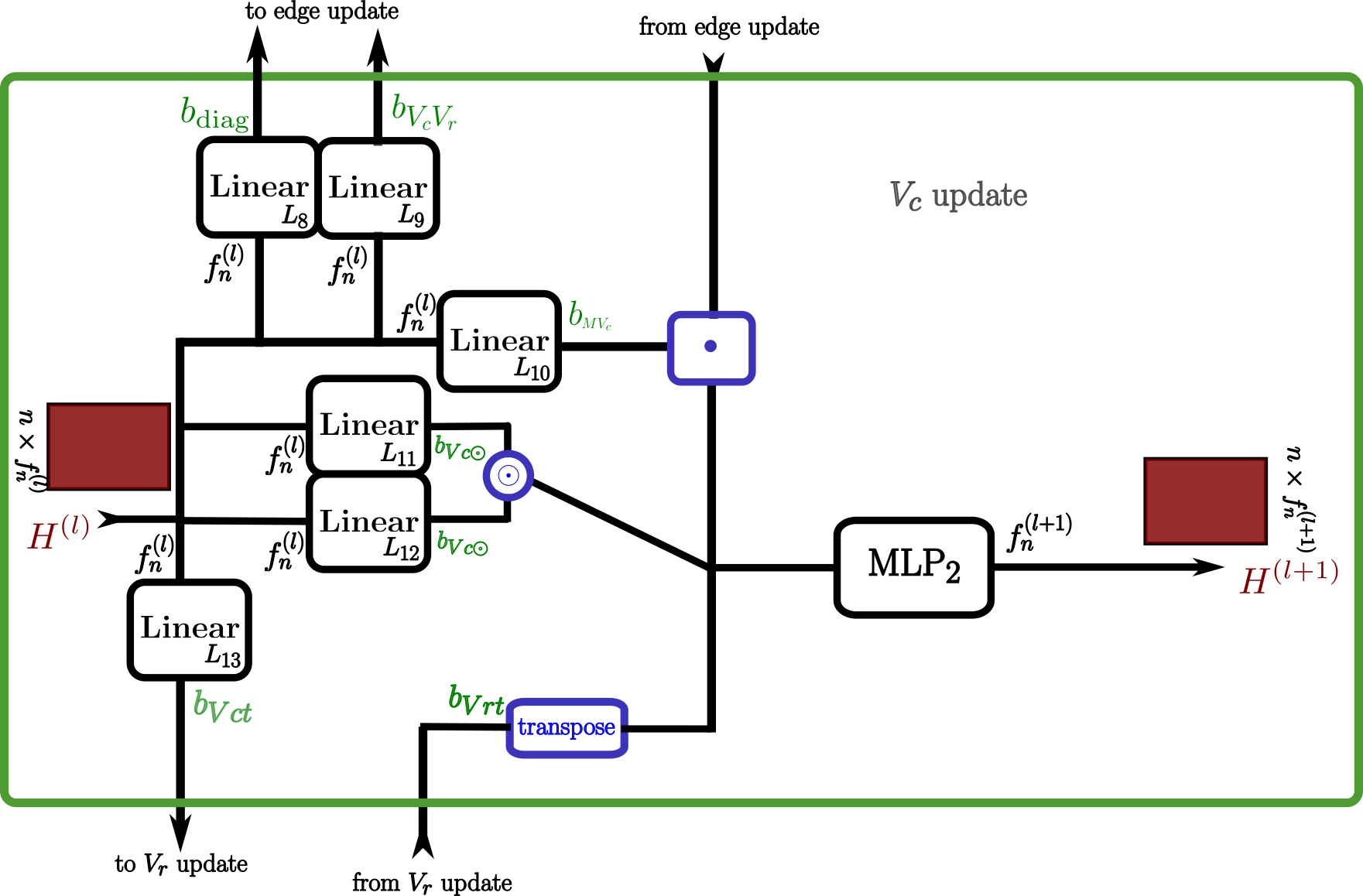}
     \includegraphics[scale = .14]{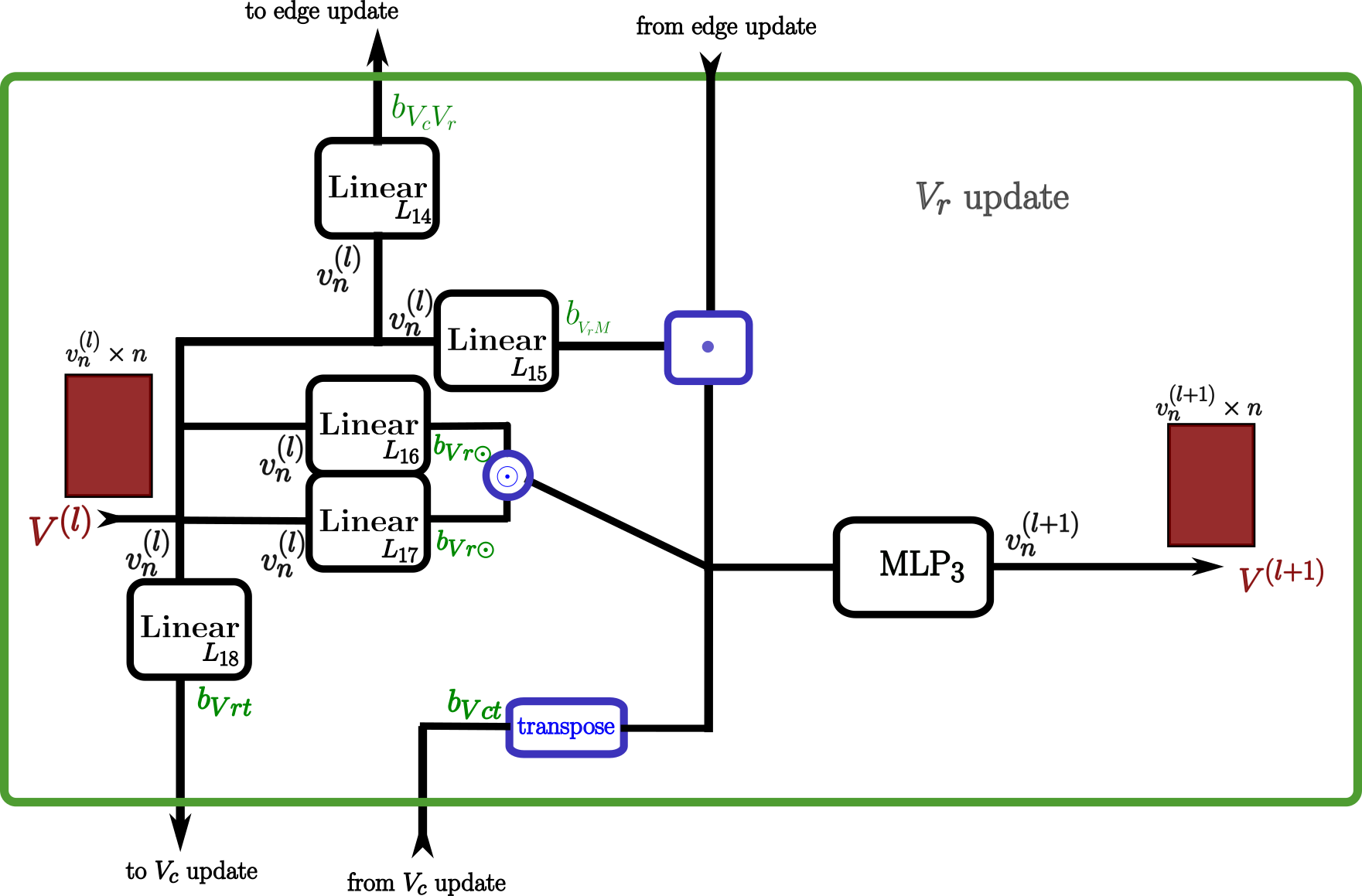}
     \caption{Layer of a GNN derived from $G_{\cur L_3}$.}
     \label{fig:exhaust_gnn}
 \end{figure}

 \begin{figure}
     \centering
     \includegraphics[scale = .14]{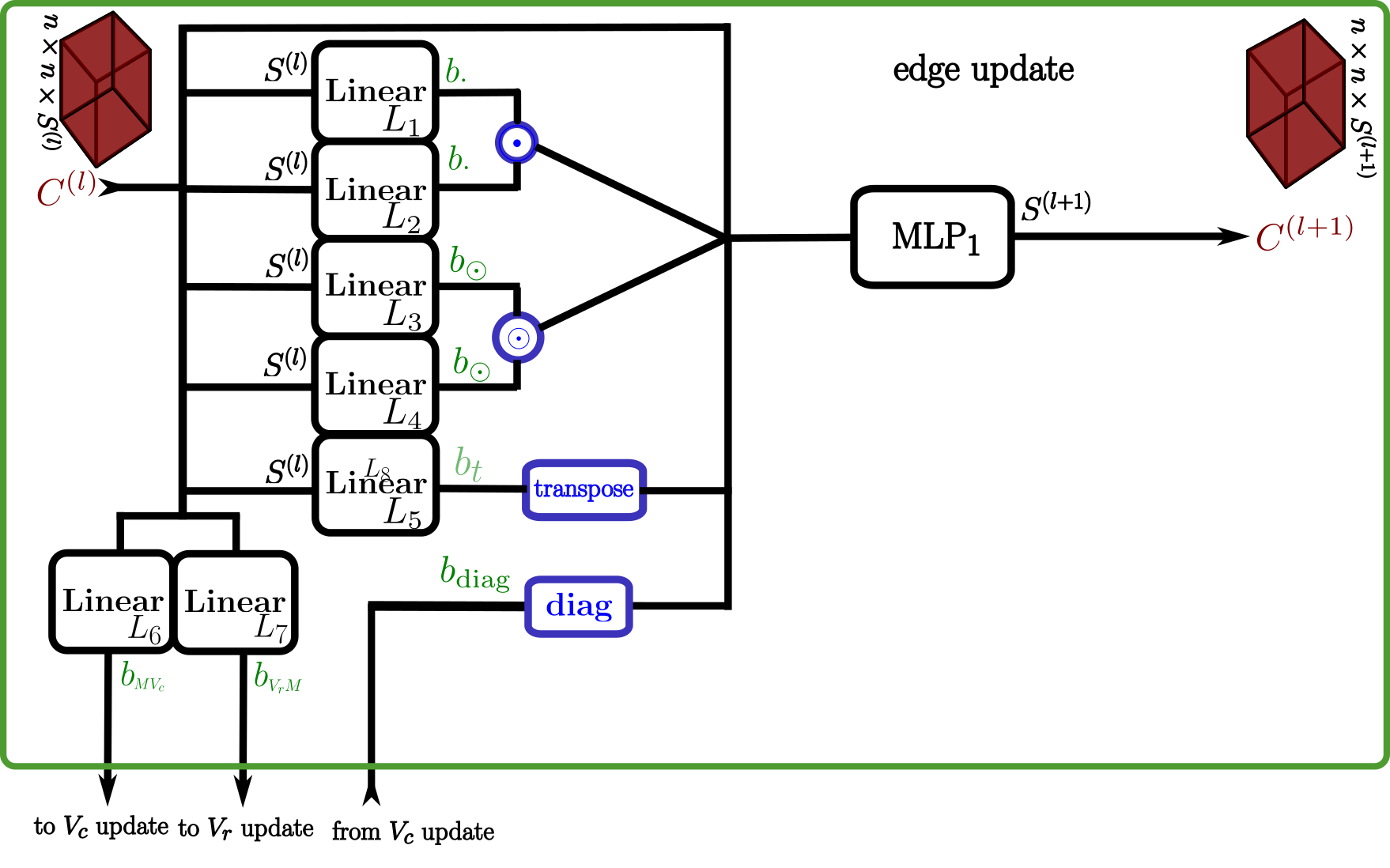}
     
     \includegraphics[scale = .14]{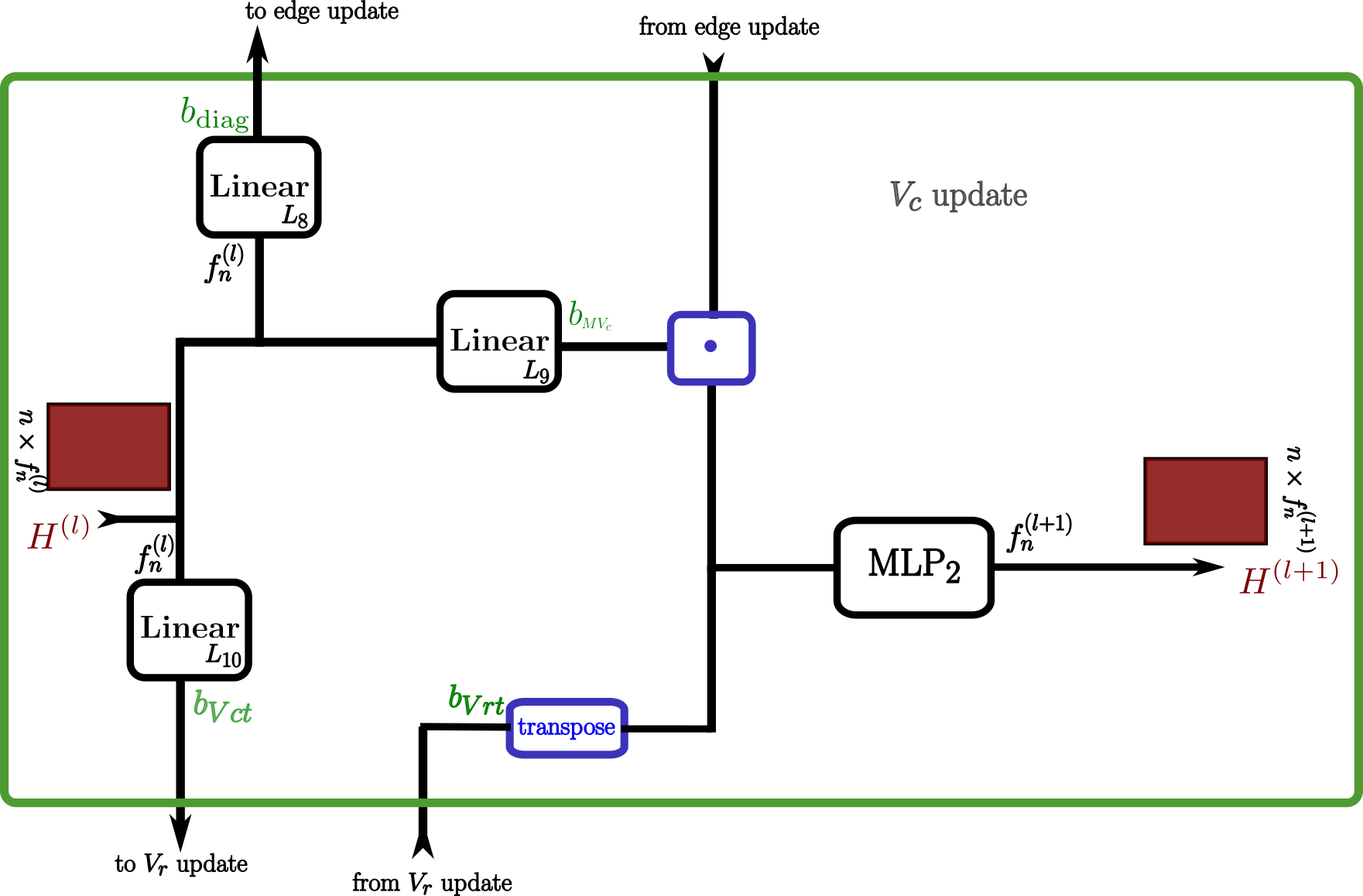}
     \includegraphics[scale = .14]{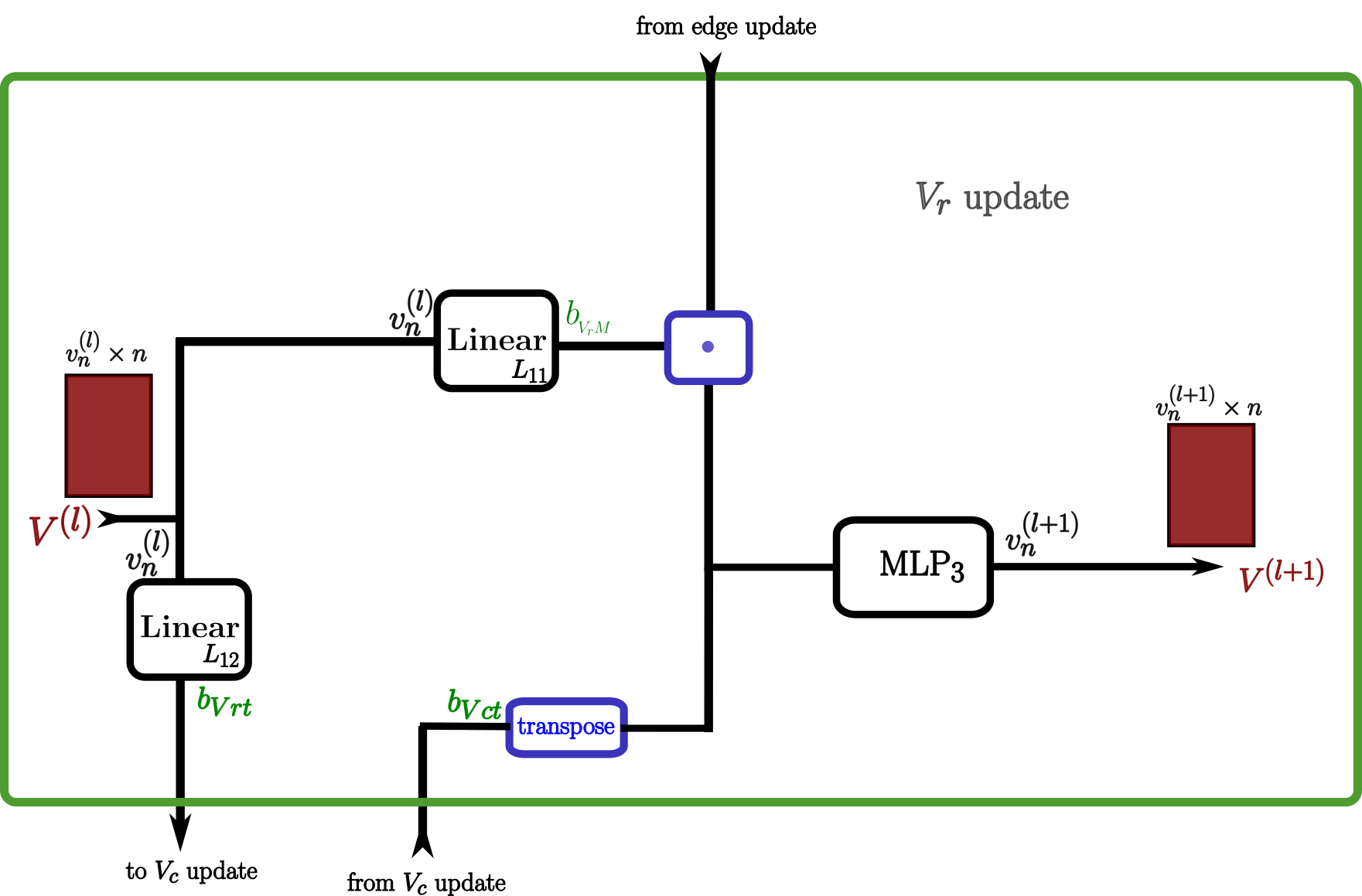}
     \caption{Layer of a GNN derived from i-$G_{\cur L_3}$.}
     \label{fig:interm_gnn}
 \end{figure}

 \begin{figure}
     \centering
     \includegraphics[width = \textwidth]{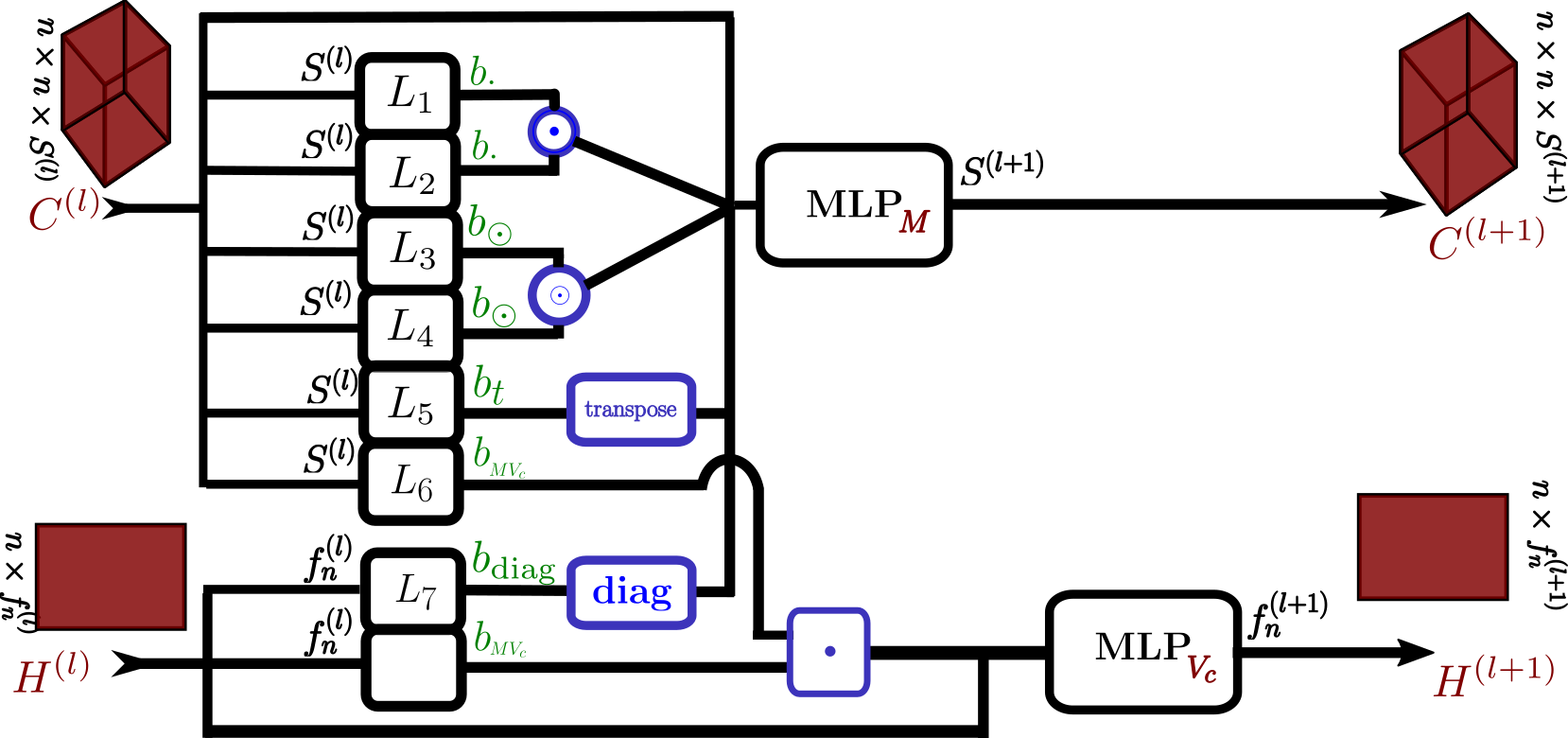}
     \caption{Layer of a GNN derived r-$G_{\cur L_3} \cup \{ \transpose M \}$.}
     \label{fig:gnn_with_transpose}
 \end{figure}

\section{Spectral response of $\ML{\cur L_3}$}\label{sec:spectral}
The graph Laplacian is the matrix $L = D - A $ (or $L = \identite - D^{-\frac 1 2}AD^{-\frac 1 2}$ for the normalised Laplacian) where $D$ is the diagonal degree matrix. Since $L$ is positive semidefinite, its eigendecomposition is $L = U\diag \lambda \transpose U$ with $U \dans \R^{n \fois n}$ orthogonal and  $\lambda \dans R^n_+$. 
By analogy with the convolution theorem, one can define graph filtering in the frequency domain by $\tilde x = U\diag{\Omega(\lambda)}\transpose U x$ where $\Omega$ is the filter applied in the spectral domain.

\begin{lem}\label{lem: laplacian from adjacency}
Given $A$ the adjacency matrix of a graph, $\ML{\cur L_3}$ can compute the graph Laplacian $L$ and the normalised Laplacian $L_n$ of this graph.
\end{lem}
\begin{proof}
$\ML{\cur L_3}$ can produce $A^2\odot \identite$ which is equal to $D$. Thus it can compute $L = D - A$.
For the normalised Laplacian, since the point-wise application of a function does not improve the expressive power of $\ML{\cur L_3}$ (\cite{Geerts}), $D^{-\frac 1 2}$ is reachable by $\ML{\cur L_3}$. Thus, the normalised Laplacian $D^{-\frac 1 2}LD^{-\frac 1 2}$ can be computed.
\end{proof}

 As in \cite{balcilarexpresspower}, we define the spectral response $\phi \dans \R^n$ of $C\dans \R^{n\fois n}$ as
$
\phi(\lambda) = \mathrm{diagonal}(\transpose U C U)
$
where $\mathrm{diagonal}$ extracts the diagonal of a given square matrix.
Using spectral response, \cite{balcilarexpresspower} shows that most existing MPNNs act as low-pass filters while high-pass and band-pass filters are experimentally proved to be necessary to increase model expressive power.



\begin{thm}\label{thm : ml3 filter}
For any continuous filter $\Omega$ in the spectral domain of the normalised Laplacian, there exists a matrix in $\ML{\cur L_3}$ such that its spectral response approximate $\Omega$.
\end{thm}
\begin{proof}
The spectrum of the normalised Laplacian is included in $\intervalleff 0 2$, which is compact. Thanks to Stone-Weierstra\ss \ theorem, any continuous function can be approximated by a polynomial function. We just have to ensure the existence of a matrix in $\ML{\cur L_3}$ such that its spectral response is a polynomial function.

For $k \dans \N$, the spectral response of $L^k$ is $\lambda ^k$ since we have
\begin{align*}
\transpose U L^k U &= \transpose U (U \diag \lambda \transpose U)^k U \\
    &= \transpose U U \diag \lambda ^k  \transpose U U= \diag \lambda ^k
\end{align*}
From Lemma \ref{lem: laplacian from adjacency}, $\ML{\cur L_3}$ can compute $L$, and thus it can compute $L^k$ for any $k\dans \N$.
Since $\ML{\cur L_3}$ can produce all the matrices with a monome spectral response and since the function that gives the spectral response to a given matrix is linear, $\ML{\cur L_3}$ can produce any matrices with a polynomial spectral response.
\end{proof}

This section shows that a $\WL 3$ GNN should be able to approximate any type of filter.

\section{Experiments}\label{subsec:setting}

\subsection{Experimental setting}
In the experiments, all  the linear blocks of a layer are set at the same width  $S^{(l)} = b_\otimes^{(l)} = b_\odot^{(l)}=b_{\mathrm{diag}}^{(l)}$. This means that MLP$_M^{(l)}$ takes as input a third order tensor of dimensions $n\times n \times 4S^{(l)}$ and MLP$_{V_c}^{(l)}$ takes as input a matrix of dimensions $n \times 2S^{(l)}$.  At each layer, the MLP depth is always 2 and the intermediate layer doubled the input dimension. 




\subsection{QM9}
For this experiment, there are 4 edge attributes and  11 node features. We use 3 layers with $S^{(l)} = f_n^{(l)} = 64$ when learning one target at a time and $S^{(l)} = f_n^{(l)} = 32$ in the other experiment for $l \dans \{1,2,3\}$. The vector readout function is a sum over the components of $H^{(3)}$ and the matrix readout function is a sum over the components of the diagonal and the off-diagonal parts of $\cur C^{(3)}$. Finally, 3 fully connected layers, with respective dimension($512/256/(1 \text{ or } 12)$) are applied before using an absolute error loss. 
Complete results on this dataset can be found in Table \ref{tab:QM9completeone}.
\begin{table*}[h]
    \centering
    \caption{Results on QM9 dataset predicting each target at a time. The metric is MAE, the lower, the better.}
    \footnotesize{
    \begin{filecontents*}{data/EXPQM9GMN.txt}
Target,1-GNN,1-2-3-GNN,DTNN,Deep LRP,NGNN,I$^2$-GNN,PPGN,G$^2$N$^2$
$\mu$,0.493,0.476,0.244,0.364,0.428,0.428,0.0934,\textbf{0.0703}
$\alpha$,0.78,0.27,0.95,0.298,0.29,0.230,0.318,\textbf{0.127}
$\epsilon_{\mathrm{homo}}$,0.00321,0.00337,0.00388,0.00254,0.00265,0.00261,{0.00174},\textbf{0.00172}
$\epsilon_{\mathrm{lumo}}$,0.00355,0.00351,0.00512,0.00277,0.00297,0.00267,{0.0021},\textbf{0.00153}
$\Delta\epsilon$,0.0049,0.0048,0.0112,0.00353,0.0038,0.0038,{0.0029},\textbf{0.00253}
$R^2$,34.1,22.9,17.0,19.3,20.5,18.64,3.78,\textbf{0.342}
ZPVE,0.00124,0.00019,0.00172,0.00055,0.0002,{0.00014},0.000399,\textbf{0.0000951}
$U_0$,2.32,0.0427,2.43,0.413,0.295,0.211,{0.022},\textbf{0.0169}
$U$,2.08,0.111,2.43,0.413,0.361,0.206,{0.0504},\textbf{0.0162}
$H$,2.23,0.0419,2.43,0.413,0.305,0.269,{0.0294},\textbf{0.0176}
$G$,1.94,0.0469,2.43,0.413,0.489,0.261,{0.024},\textbf{0.0214}
$C_v$,0.27,0.0944,2.43,0.129,0.174,0.0730,0.144,\textbf{0.0429}
\end{filecontents*}
\csvautobooktabular{data/EXPQM9GMN.txt}
    }
    \label{tab:QM9completeone}
\end{table*}

\subsection{TUD}
The parameter setting for each of the 6 experiments related to this dataset can be found in Table \ref{tab:parmtud}. Complete results on this dataset are given in Table \ref{tab:TUDcomplete}.
\begin{table*}[h]
    \centering
    \caption{G$^2$N$^2$ parameters detail for each dataset in our experiments on TU}
    \footnotesize{\begin{filecontents*}{data/paramTUD.txt}
parameters,MUTAG,PTC,Proteins,NCI1,IMDB-B,IMDB-M
node features,7,22,3,37,1,1
edge features,6,0,0,0,0,0
$\# \text{ of } \ggnn \text{ layer }= l_m$,3,3,3,3,3,3
$f_n^{(l)} \ l\dans \intervalleentier 1 {l_m}$,16,32,32,64,32,32
$S^{(l)} \ l\dans \intervalleentier 1 {l_m}$,16,32,32,64,32,32
readout dimension, 256/128/1,512/256/1,512/256/1,128/64/1,512/256/1,512/256/3
loss,BCEloss,BCEloss,BCEloss,BCEloss,BCEloss,CEloss
\end{filecontents*}
\csvautobooktabular{data/paramTUD.txt}}
    \label{tab:parmtud}
\end{table*}
\begin{table*}[h]
    \centering
    \caption{Results on TUD dataset. The metric is accuracy, the higher, the better.}
    \footnotesize{
    \begin{filecontents*}{data/TUD.txt}
Dataset,MUTAG,PTC,Proteins,NCI1,IMDB-B,IMDB-M
WL kernel (\cite{shervashidze2011weisfeiler}),90.4$\pm$5.7,59.9$\pm$4.3,75.0$\pm$3.1,\textbf{86.0}$\pm$\textbf{1.8},73.8$\pm$3.9,50.9$\pm$3.8
GNTK (\cite{du2019graph}),90.0$\pm$8.5,67.9$\pm$6.9,75.6$\pm$4.2,84.2$\pm$1.5,76.9$\pm$3.6,52.8$\pm$4.6
DGCNN (\cite{zhang2018end}),85.8$\pm$1.8,58.6$\pm$2.5,75.5$\pm$0.9,74.4$\pm$0.5,70.0$\pm$0.9,47.8$\pm$0.9
IGN (\cite{maron2018invariant}),83.9$\pm$13.0,58.5$\pm$6.9,76.6$\pm$5.5,74.3$\pm$2.7,72.0$\pm$5.5,48.7$\pm$3.4
GIN (\cite{xu2018powerful}),89.4$\pm$5.6,64.6$\pm$7.0,76.2$\pm$2.8,82.7$\pm$1.7,75.1$\pm$5.1,52.3$\pm$2.8
PPGNs (\cite{maron2019provably}),90.6$\pm$8.7,66.2$\pm$6.6,77.2$\pm$4.7,83.2$\pm$1.1,73.0$\pm$5.8,50.5$\pm$3.6
Natural GN (\cite{de2020natural}),89.4$\pm$1.60,66.8$\pm$1.79,71.7$\pm$1.04,82.7$\pm$1.35,74.8$\pm$2.01,51.3$\pm$1.50
WEGL (\cite{kolouri2020wasserstein}),N/A,67.5$\pm$7.7,76.5$\pm$4.2,N/A,75.4$\pm$5.0,52.3$\pm$2.9
GIN+GraphNorm (\cite{cai2021graphnorm}),91.6$\pm$6.5,64.9$\pm$7.5,77.4$\pm$4.9,82.7$\pm$1.7,76.0$\pm$3.7,N/A
GSNs (\cite{bouritsas2022improving}),{92.2}$\pm${7.5},68.2$\pm$7.2,76.6$\pm$5.0,83.5$\pm$2.0,\textbf{77.8}$\pm$\textbf{3.3},\textbf{54.3}$\pm$\textbf{3.3}
G$^2$N$^2$,\textbf{92.5}$\pm$\textbf{4.3},\textbf{72.3}$\pm$\textbf{6.3},\textbf{80.1}$\pm$\textbf{3.7},82.8$\pm$0.9,76.8$\pm$2.8,54.0$\pm$2.9
\end{filecontents*}
\csvautobooktabular{data/TUD.txt}
    }
    \label{tab:TUDcomplete}
\end{table*}

\subsection{Spectral dataset}

This dataset is composed of three 2D grids of size 30x30, for respectively training, validation, and testing. We use 3 layers of G$^2$N$^2$ with $S^{(l)}=32$ and  $f_n^{(l)}=32$ for $l \dans \{1,2,3\}$. Our readout function is the identity over the last node embedding and a sum over the line of the last edge embedding. We finally apply two fully connected layers on the output of the readout function and then use Mean Square Error (MSE) loss to compare the output to the ground truth.

\end{document}